\documentclass[manuscript, screen, nonacm]{jair}
\AtBeginDocument{%
    \fancyfoot{}

    \fancypagestyle{firstpagestyle}{%
      \fancyhf{}%
    }%
}
\setcopyright{none}

\JAIRTrack{}
\acmVolume{0}
\acmArticle{0}
\acmMonth{8}
\acmYear{2026}

\RequirePackage[
  datamodel=acmdatamodel,
  style=acmauthoryear,
  backend=biber,
  giveninits=true,
  uniquename=init
  ]{biblatex}

\usepackage[ruled,vlined,linesnumbered]{algorithm2e}
\usepackage{tikz}
\usetikzlibrary{arrows.meta,positioning,fit,backgrounds,calc,shapes.geometric,
                decorations.pathreplacing}
\usepackage{enumitem}
\usepackage{multirow}
\usepackage{array}

\definecolor{seagreen}{HTML}{0F766E}
\definecolor{rust}{HTML}{B45309}
\definecolor{claret}{HTML}{9B1C31}

\theoremstyle{acmdefinition}
\newtheorem{principle}{Principle}
\newtheorem{criterion}{Criterion}
\makeatletter
\@ifundefined{remark}{\newtheorem{remark}{Remark}}{}
\makeatother

\SetKwFunction{Rank}{ScoreAndServe}
\SetKwFunction{Cover}{GreedyCover}
\SetKwFunction{Adopt}{PickFromSlate}
\SetKwFunction{Compose}{WriteNew}
\SetKwFunction{Degen}{WriteWithoutReason}
\SetKwFunction{LinkOut}{AssertLinks}
\SetKwFunction{Merge}{MergeSigned}
\SetKwInput{Given}{Given}
\SetKwInput{Yields}{Yields}

\newcommand{\ABAS}{\textsc{Abas}}
\newcommand{\DDP}{\textsc{DDP2P}}
\newcommand{\Jpos}{\mathcal{J}^{+}}
\newcommand{\Jneg}{\mathcal{J}^{-}}
\newcommand{\Rreb}{\mathcal{R}^{\ominus}}
\newcommand{\Rrei}{\mathcal{R}^{\oplus}}
\newcommand{\wc}{\omega_{\mathrm{c}}}
\newcommand{\wrel}{\omega_{\mathrm{r}}}
\newcommand{\lab}{\lambda}
\newcommand{\Lam}{\Lambda}
\newcommand{\pown}{\pi_{\mathrm{own}}}
\newcommand{\preuse}{\pi_{\mathrm{reuse}}}
\newcommand{\pabst}{\pi_{\mathrm{abst}}}
\newcommand{\plnk}{\pi_{\mathrm{lnk}}}

\newcommand{\slate}{\mathcal{S}}
\newcommand{\AUC}{\mathrm{AUC}}
\newcommand{\RDC}{\mathrm{RDC}}
\newcommand{\mass}{\mathrm{M}}

\begin{document}

\title[Rules Before Oracles]{Rules Before Oracles: Auditable,
User-Configurable Argument Selection for Deliberative Polling}

\author{Muntaser Syed}
\email{muntaser@my.fit.edu}
\affiliation{%
  \institution{Florida Institute of Technology}
  \city{Melbourne}
  \state{Florida}
  \country{USA}
}

\author{Markus Zanker}
\email{markus.zanker@unibz.it}
\affiliation{%
  \institution{Free University of Bozen-Bolzano}
  \city{Bolsano}
  \country{Italy}
}

\author{Marius Silaghi}
\authornote{Corresponding Author.}
\email{msilaghi@fit.edu}
\affiliation{%
  \institution{Florida Institute of Technology}
  \city{Melbourne}
  \state{Florida}
  \country{USA}
}

\renewcommand{\shortauthors}{Syed\&Zanker\&Silaghi}

\begin{abstract}
{\bf Background:}
A deliberative poll asks people to decide only after considering the arguments
that bear on the decision. Once submitted arguments outnumber what anyone will
read, some mechanism must choose which of them each voter sees, and that
choice acquires a large share of the decision. Contemporary practice delegates
it to opaque learned rankers, so a participant cannot recompute, attribute or
contest the exposure that shaped their vote.

{\bf Objectives:}
We ask whether argument selection in a binding civic process can be a
\emph{published rule} over \emph{publicly recomputable evidence}, with free
parameters held by the individual voter, and what such a constraint costs. We
treat legibility/explainability as an admissibility condition on the class of usable
mechanisms rather than as an objective to be traded against accuracy, and we
quantify the trade-off that is thereby foreclosed.

{\bf Methods:}
We formalise a poll as an alternative-based information system over bipolar
justification sets, define three complementary evaluation instruments for judging a
served slate: (i) set coverage of the live reason vocabulary, (ii) the order in
which coverage arrives, and (iii) captured endorsement mass. We also state the
Subsuming Justification Problem with a greedy submodular bound used strictly
as an evaluation ceiling, and give seven checkable criteria for a civic
recommender and a rule satisfying all seven: a one-hop reversed endorsement
flow whose only policy parameter is a relation-weight function. An agentic
simulator instantiates $N$ constituent agents and persists every slate at the
instant of every vote; we report roughly $17{,}000$ seeded, seed-paired runs
over two propositions.

{\bf Results:}
In terms of coverage of live reason vocabulary fraction,
served slates fall only $0.035 \pm 0.013$ short of a label-reading ceiling
that upper-bounds \emph{every} selection procedure, opaque ones included,
so the entire competitive advantage available to an unconstrained ranker is
bounded and small. On set coverage alone with solely non-degenerate authoring the rule is statistically
indistinguishable from a uniformly random slate; we show that null is an artefact of an order-blind, charity-blind instrument.
Under
the two remaining evaluation instruments, the rule has a clear effect: it leads at every slate
prefix by a margin that \emph{widens} with adversarial pressure ($-5.7$
positions of depth-to-90\% at a quarter-electorate coalition), and it
dominates on endorsement mass by a factor of $3.3$. Once a realistic fraction
of submissions carries no reasons, the coverage margin returns and grows
monotonically, with the two link summands --- exactly inert on a uniformly
good corpus --- supplying $91\%$ of it. Label-homogeneous flooding collapses
completeness from $0.81$ to $0.34$ under a flat weight policy but only to
$0.44$ under an author-count-normalised one; a coalition co-signing to defeat
the normalisation makes itself monotonically weaker.

{\bf Conclusions:}
Serving as a key security control, the weight policy carries significant completeness value (10\%).
The choice between ranking arms is a position on a coverage-versus-mass frontier rather than a fact, which is precisely the kind of choice only a legible rule can hand to the person it affects.
We map the construction onto
an existing open-source peer-to-peer platform, where per-peer local evaluation
turns configurability from an operator's concession into a structural
property.
\end{abstract}

\maketitle

\section{Introduction}
\label{sec:intro}

In an assembly small enough that everyone hears everything, nobody has to
decide what gets heard. Every real electorate is larger than that, so a
selection step is unavoidable: something decides which of the thousands of
submitted reasons appear on the screen of a voter about to cast a ballot.
That step is where the power sits. The tally is public and checkable; the
slate of arguments that produced the votes is usually neither. This manuscript
takes the position that in a binding civic process the selection step is a
piece of democratic procedure and must be built like one --- a rule published
in advance, computed over evidence any participant can retrieve and
recompute, with the free parameters belonging to the individual voter rather
than to whoever runs the servers. The key insight that makes this practical is
that the property we actually want from a slate --- that its reasons jointly
span the reasons in play --- is a \emph{set-level} property of the served
collection, and set-level properties can be pursued through the structure of
the argumentation graph without ever interpreting the text; a rule that reads
only endorsement counts and signed links is therefore both good enough to use
and far harder to bend without leaving a public trace. Technically we model a
poll as an alternative-based information system, measure a served slate along
three axes rather than one, exhibit a one-hop reversed endorsement-flow rule
whose sole policy lever is a relation-weight function, evaluate it inside a
seed-controlled agentic simulator against a greedy submodular ceiling and
under coordinated attack, and map the result onto an existing peer-to-peer
deliberation platform. The contributions are:
\begin{itemize}[leftmargin=*,itemsep=1pt]
  \item \textbf{A charter for civic recommenders} (Section~\ref{sec:charter}):
    address seven operational criteria --- 1. determinism, 2. evidence locality, 3. author
    blindness, 4. semantic abstinence, 5. reproducibility, 6. contestability,
    7. configurability --- each with the test that discharges it and the
    observable that would show it failing, argued as admissibility conditions
    rather than objectives; an opaque learned ranker fails four by
    construction.
  \item \textbf{A three-instrument model of what a slate is for}
    (Section~\ref{sec:model}): define an alternative-based poll over bipolar
    justification sets and measure 1. completeness of coverage against the reason vocabulary live at the
    instant of each vote; 2. a prefix-sensitive family of order measures; 3. captured
    endorsement mass. The Subsuming Justification Problem is stated with its weighted and
    componentwise variants, and a greedy $(1-e^{-1})$ cover to be used strictly as
    an evaluation ceiling.
  \item \textbf{A rule that meets the charter}
    (Sections~\ref{sec:charter}--\ref{sec:abas}): a one-hop reversed
    endorsement flow over rebuttal and reinforcement links, with a per-voter
    policy vector making the rule individually configurable without making it
    individually unpredictable; and \ABAS{}, a seed-controlled simulator of
    $N$ constituent agents that persists every served slate before the ballot
    it informed, yielding some $17{,}000$ runs across sensitivity sweeps, a
    degenerate-authoring sweep, ranking-term ablations and five adversarial
    families.
  \item \textbf{A measured price for the charter}
    (Section~\ref{sec:res_ceiling}): the coverage fraction distance between the served slates and
    a label-reading greedy ceiling that upper-bounds every selection procedure
    on the same pool is $0.031 \pm 0.014$, against a temporal component of
    $0.163$ that no procedure of any kind could recover.
  \item \textbf{Experimental results and what a ranking rule is actually for}
    (Sections~\ref{sec:instruments}--\ref{sec:res_pareto}): on set coverage
    alone, with solely non-degenerate authoring, the rule does not separate from a uniformly random slate, its two
    link summands are all but inert, and under flooding the random slate even wins. We show it is an
    artefact of an order-blind instrument and an implausibly charitable
    authoring model: under an order-sensitive reading the rule leads at every
    prefix with the margin growing under attack, and once a realistic fraction
    of submissions fails to justify, the coverage margin returns, rises
    monotonically with that fraction, and makes the link summands worth as much
    as the whole ranking advantage --- by squaring the discrimination the
    electorate already exercises.
  \item \textbf{An empirical manipulation result with a named defence}
    (Section~\ref{sec:adversarial}): under authentication, hub-riding is
    harmless and label-homogeneous flooding is not; author-count normalisation
    retains $0.08$--$0.12$ of completeness that a flat policy loses; a rate-
    and corpus-matched control isolates homogeneity as four fifths of the
    cause; and the obvious escalation --- co-signing links to inflate the
    normalising numerator --- makes the coalition monotonically weaker.
  \item \textbf{A peer-to-peer realisation and a normative reading}
    (Sections~\ref{sec:p2p} and~\ref{sec:ecitizen}): a component-by-component
    mapping onto DirectDemocracy Peer-to-Peer (\DDP{}), covering identity and census, gossip synchronisation
    and local evaluation over partial replicas; and an argument for why legible
    procedure is constitutive rather than decorative.
\end{itemize}
The remainder of this section states the exposure problem and previews what the
measurements turned out to say. Section~\ref{sec:related} reviews the
literatures the work sits between, Section~\ref{sec:model} builds the formal
object and its three instruments, Section~\ref{sec:charter} states the charter
and the rule, Section~\ref{sec:abas} describes the simulator and
Section~\ref{sec:setup} the protocol.
Sections~\ref{sec:res_honest}--\ref{sec:adversarial} report the non-degenerate authoring attack-free
sweeps and the price of semantic abstinence, the null result and its two
resolutions, and the adversarial sweeps. Section~\ref{sec:p2p} maps the design
onto \DDP{}, Section~\ref{sec:discussion} discusses,
Section~\ref{sec:limits} lists what would change the conclusions,
Section~\ref{sec:ecitizen} makes the normative argument, and
Section~\ref{sec:conclusion} concludes.

\subsection{The exposure problem, and the position taken here}
\label{sec:exposure}

The legitimacy of a collective decision has never rested on the tally alone.
From Mill's marketplace of ideas to Habermas's account of communicative
rationality~\cite{mill1859liberty,habermas1984theory}, the tradition holds that a decision
earns authority from the quality of the discourse preceding it. Fishkin turned
that into a measurable
protocol~\cite{fishkin1991democracy,fishkin2000deliberative}: draw a
stratified sample, expose it to balanced material and structured discussion,
measure how opinion moves. Deployments across many countries report the same
qualitative finding --- considered opinion differs systematically from raw
opinion, and partisan cues lose their grip once people meet the reasons behind
positions~\cite{luskin2002considered,fishkin2009whenpeople} --- and the
protocol's weakness is cost, a facilitated cohort of a few hundred being an
expensive instrument that does not obviously scale to a national
electorate~\cite{lukensmeyer2005largescale}.

Digital platforms remove that barrier and immediately install another. When a
hundred thousand people submit reasons, nobody reads the corpus; each person
reads a slate. Call this the \emph{exposure problem}: given a growing corpus
of justifications and a per-voter display budget of $K$ items, how should the
slate be chosen so that the reasons in play are represented in what voters
actually see? Getting it wrong produces three distinct {\bf harms}, and they call
for different remedies: 
\begin{itemize}
\item 
\emph{Temporal inequity}: in a sequential process the
corpus is thin at the start and rich at the end, so early voters decide
against a narrower reason space than late voters through no fault of their
own; at the reference configuration the spread of per-voter completeness
\emph{within} a single run is about $0.21$, five times the variation between
one entire run and another. 
\item
\emph{Silent narrowing}: a selector tuned for
engagement converges on whatever holds attention, which in political material
is typically what confirms a
prior~\cite{pariser2011filter,sunstein2007republic,willson2014politics}, and
the narrowing is silent because no individual slate looks censored --- the
missing reasons are simply never surfaced. 
\item
\emph{Unfalsifiable influence}: if
the selector is a learned model, a participant who suspects their side is
being systematically under-surfaced cannot establish it, cannot recompute the
ranking, cannot point at the clause that produced the outcome, and cannot
distinguish deliberate suppression from an artefact of training
data~\cite{burrell2016machine,lipton2018mythos}. A civic process unable to
answer ``why was I shown this?'' with a checkable derivation has replaced
procedural legitimacy with trust in an operator.
\end{itemize}

The position of this manuscript follows from the last two harms in particular, and
it is a position about \emph{admissibility}, not about performance. We do not
claim that opaque rankers rank badly. We claim that opaque ranking is the wrong frame
for a civic slate: the object produced is not a personalised feed but a piece
of the public record of what a voter was shown, and a public record must be
reconstructible. Legibility therefore determines which procedures may be used
at all, and questions of comparative quality arise only among those that
qualify --- the same structure as the secret ballot, which nobody defends on
the grounds that it measures preferences more accurately than open voting.
Section~\ref{sec:desiderata} states the criteria in a form that can be checked
against an implementation, and every subsequent section is written so that a
reader can see which criterion each design decision is discharging.
Figure~\ref{fig:pipeline} states the commitment in one picture.
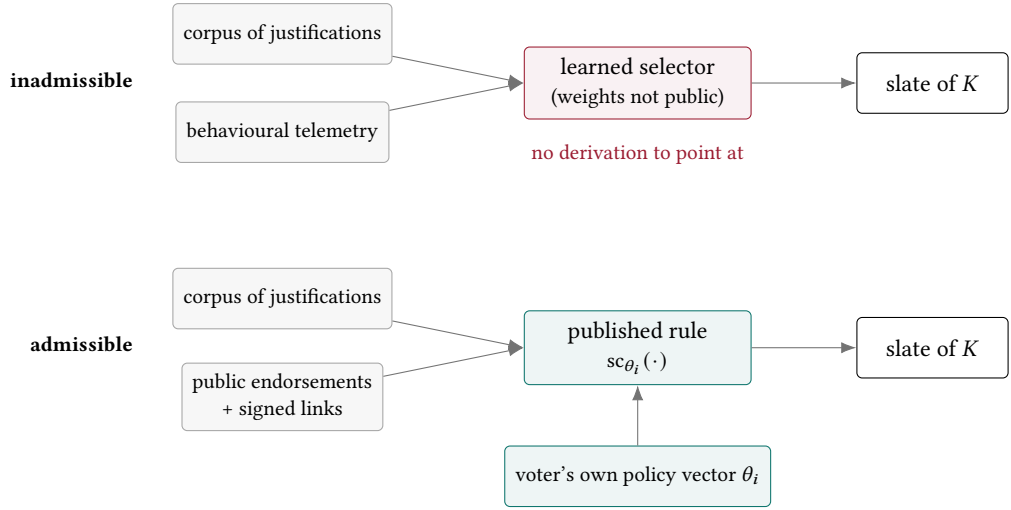
\begin{figure}[tb]
\centering
\begin{tikzpicture}[
    font=\small,
    node distance=5mm,
    box/.style={draw,rounded corners=2pt,align=center,inner sep=4pt,
                minimum height=8mm,fill=white},
    good/.style={box,draw=seagreen,fill=seagreen!7},
    bad/.style={box,draw=claret,fill=claret!6},
    ev/.style={box,draw=black!35,fill=black!3,font=\footnotesize},
    ar/.style={-{Latex[length=2mm]},draw=black!55}
  ]
  \node[ev,minimum width=26mm] (c1) at (0,0.65) {corpus of justifications};
  \node[ev,minimum width=26mm] (b1) at (0,-0.65) {behavioural telemetry};
  \node[bad,minimum width=30mm] (m1) at (4.7,0)
       {learned selector\\{\footnotesize (weights not public)}};
  \node[box,minimum width=20mm] (s1) at (8.6,0) {slate of $K$};
  \draw[ar] (c1) -- (m1.west);
  \draw[ar] (b1) -- (m1.west);
  \draw[ar] (m1) -- (s1);
  \node[left=4mm of c1.west,anchor=east,yshift=-6mm,align=right,
        font=\footnotesize\bfseries] {inadmissible};
  \node[below=2mm of m1,font=\footnotesize,text=claret]
       {no derivation to point at};

  \node[ev,minimum width=26mm] (c2) at (0,-2.85) {corpus of justifications};
  \node[ev,minimum width=26mm] (b2) at (0,-4.15)
       {public endorsements\\+ signed links};
  \node[good,minimum width=30mm] (m2) at (4.7,-3.5)
       {published rule\\{\footnotesize $\mathrm{sc}_{\theta_i}(\cdot)$}};
  \node[box,minimum width=20mm] (s2) at (8.6,-3.5) {slate of $K$};
  \node[good,minimum width=30mm,font=\footnotesize] (p2) at (4.7,-5.2)
       {voter's own policy vector $\theta_i$};
  \draw[ar] (c2) -- (m2.west);
  \draw[ar] (b2) -- (m2.west);
  \draw[ar] (m2) -- (s2);
  \draw[ar] (p2) -- (m2);
  \node[left=4mm of c2.west,anchor=east,yshift=-6mm,align=right,
        font=\footnotesize\bfseries] {admissible};
\end{tikzpicture}
\caption{The two pipelines. Both consume the same corpus and both emit $K$
items. The upper one additionally consumes behavioural telemetry and passes it
through parameters nobody outside can inspect, so a disputed slate has no
derivation to point at. The lower one consumes only evidence every participant
can already retrieve --- who endorsed what, and which signed links exist ---
and exposes its policy parameters to the individual voter, so a disputed slate
reduces either to a disputed input, which is checkable, or to a disputed
policy, which is arguable. This manuscript treats the difference as a
condition of admissibility rather than as a term in an objective.}
\label{fig:pipeline}
\end{figure}

\subsection{What the measurements turned out to say}
\label{sec:arc}

In a simplified setting where all justification authors correctly express their reasons (non-degenerate authoring), the fraction of the live reason vocabulary covered by a served slate, a simplified optimization objective for the motivating harms, does not differ significantly between our rule and random selection.

With metrics that evaluate the ordering of justifications within the slate, and/or under a realistic electorate with degenerate authoring, the rule shows strong benefits over the control alternatives.

\subsection{Why the substrate matters}
\label{sec:whyp2p}

A published rule evaluated on a server the operator controls is a large
improvement over an unpublished one, and it stops short. The operator still
decides which endorsements are in the database, when the index refreshes, and
whose items quietly stop being returned. Configurability granted by an
operator is revocable by the same operator, and reproducibility asserted by an
operator is a claim about a machine nobody else can inspect. The natural
terminus of the argument is an architecture in which each participant holds
their own replica of the items they care about and evaluates the rule locally
over it. \DDP{} --- \textsc{DirectDemocracyP2P} --- is an existing open-source
Java platform built on precisely that
premise~\cite{silaghi2013ddp2p,silaghi2013petition}: peers keep independent
databases of self-contained items, name them by global identifiers derived
from public keys and content digests, and converge by push--pull gossip.
Section~\ref{sec:p2p} shows the model developed here maps onto \DDP{}'s item
types with very little impedance mismatch, and that doing so converts several
criteria of the charter from policy promises into structural facts.
\section{Background and Related Work}
\label{sec:related}

This work sits at the junction of five literatures:
(1) the empirical political science of deliberative polling, (2) the formal
theory of bipolar argumentation, (3) computational social choice for participatory
processes, (4) the systems literature on recommendation and its manipulation, and
(5) the peer-to-peer work that supplies the substrate. A sixth --- the use of
large language models as simulated populations --- supplies the register in
which the evaluation is conducted. We state below what we take from each and,
where it matters, what we deliberately do not take.

\paragraph{Deliberative polling and its scaling problem.}
\label{sec:rw_dp}
Deliberative polling instruments a simple hypothesis: opinion formed after
exposure to balanced argument differs systematically from opinion measured
cold~\cite{fishkin1991democracy,fishkin2000deliberative}. The canonical design
draws a stratified sample, supplies balanced briefing material, runs moderated
small-group discussion, and re-measures. Results across many deployments are
consistent in direction --- movement is largest where the initial position was
least informed~\cite{luskin2002considered,fishkin2009whenpeople} --- and the
effect survives when partisan cues are stripped from the
material~\cite{price2002disagreement}. The design constraint is that
facilitation does not scale: cohorts are a few hundred and the per-participant
cost is that of a small conference~\cite{lukensmeyer2005largescale}. Digital
deployments relax that constraint and inherit a new
one~\cite{tolbert2009strategic,toots2019participation}: with a corpus no
participant reads in full, the balance of the briefing material is no longer
something an organiser curates but something a selection mechanism produces,
per voter, in real time. We take the goal --- reasoned exposure before the
ballot --- and treat selection, rather than facilitation, as the object to be
engineered.

\paragraph{Bipolar argumentation.}
\label{sec:rw_arg}
Dung's abstract frameworks model a debate as a directed attack graph and
define acceptability through extensions~\cite{dung1995acceptability}. Bipolar
frameworks add a support relation alongside
attack~\cite{cayrol2005bipolarity,cayrol2013bipolarity}, and the handbook
literature surveys the resulting semantic
landscape~\cite{baroni2018handbook,amgoud2008acceptability}; labelled bipolar
frameworks give a unified semantics for the many readings of support and make
the choice among them
explicit~\cite{escanuela2021labeled}, and the correspondence with logic
programming has since been settled in
detail~\cite{alcantara2025equivalence}. Value-based extensions attach
audience-relative preferences to arguments and thereby explain rational
disagreement between audiences that accept the same
facts~\cite{benchcapon2007argumentation}, which is close in spirit to the
per-voter policy vector of Section~\ref{sec:config}.
Our use of this apparatus is deliberately shallow, and the shallowness is a
design commitment. We adopt the bipolar signature --- two relations of opposite
polarity over a set of items --- and decline the semantics: we compute no
extensions and pronounce no argument acceptable or defeated. The reason is
Criterion~\ref{c:abstain} (semantic abstinence,
Section~\ref{sec:desiderata}) --- any semantics deciding which arguments
survive is a machine deciding which arguments a voter should stop considering,
exactly the authority a civic slate must not delegate. Links here are evidence
about relevance, not adjudication of truth. That commitment also distinguishes
this work from dialectical-quality measures over exchanges between
parties~\cite{rocha2026assessing}, which judge the argumentation; Prakken and
Sartor on argument schemes in law~\cite{prakken2001formalizing} and the
procedural tradition running back to Robert's rules~\cite{robert1915rules} are
closer to the role links play here, structuring who may speak to what rather
than who is right.

\paragraph{Computational social choice for participatory processes.}
\label{sec:rw_cs}
A parallel line asks how collective decisions should be composed once
participation is mediated by software. Liquid democracy has been given an
algorithmic treatment with explicit accuracy guarantees and failure
modes~\cite{kahng2021liquid}; representative committees can be sampled from
peers so that the committee's decisions track the electorate's~\cite{meir2021representative};
proposing and voting have been unified as aggregation over a metric
space~\cite{bulteau2021aggregation}; and proportionality has been extended
from one-shot elections to sequential decision
making~\cite{chandak2026proportional}. The complexity of outcome
determination in judgment aggregation is likewise
mapped~\cite{endriss2020complexity}. This literature aggregates
\emph{positions}; the present work is upstream of it, concerning what a voter
is shown before a position is formed, and the two are complementary: any of
these aggregation rules can sit downstream of the slate mechanism studied here.

\paragraph{Argument mining and its role here.}
\label{sec:rw_mining}
Argument mining extracts claims, premises and relations from
text~\cite{lippi2016argumentation,stab2017parsing}; retrieval systems rank
passages by argumentative
quality~\cite{wachsmuth2018retrieval,carenini2006argumentative}; and
argumentative dialogue agents use such structures to conduct persuasive
exchanges~\cite{chalaguine2020chatbot}. These techniques are how the reason
labels of Section~\ref{sec:tags} would be produced in a real deployment, and we
assume nothing better than what they currently deliver. Their placement in the
architecture is the point. Label extraction is an \emph{authoring-time}
operation: it runs once per item, its output is attached to the item, it is
visible to the author, and it can be contested and corrected before the item is
ever served. Selection is a \emph{serving-time} operation over that frozen
output. Keeping a language model strictly on the authoring side of that line
means a disputed slate never requires anyone to reason about model internals;
it requires them to point at a label they think is wrong, which is a claim
about a public artefact. Section~\ref{sec:limits} returns to what happens when
the labels themselves are adversarial.

\paragraph{Civic platforms with algorithmic assistance.}
\label{sec:rw_civic}
Deployed systems already make these choices. Polis clusters participants by
agreement pattern and surfaces statements that bridge
clusters~\cite{small2021polis}, an approach with a genuine claim to
representativeness whose clustering step nevertheless requires the operator's
pipeline to be trusted; hybrid participatory systems go further and estimate
the \emph{values} behind participants' textual motivations, disambiguating
them interactively~\cite{liscio2025value}, which is the closest published
treatment of the authoring-time inference we place off the serving path.
Deliberation-support platforms have been surveyed for their argumentation
structures~\cite{brenneis2021howwill,mancini2015time}, and recent work examines
language models as facilitators~\cite{behrendt2024ai,tessler2024habermas}, with
Tessler et al.\ reporting that a model-generated statement can find more common
ground than a human mediator. That result is real and we do not dispute it. Our
objection is categorical rather than empirical: a mediator whose reasoning
cannot be reconstructed is unsuitable for a binding process regardless of
measured quality, for the same reason that a demonstrably accurate but
unauditable vote count is unsuitable. The regulatory direction is consistent
--- the Digital Services Act obliges very large platforms to disclose
recommender parameters and offer a non-profiling option~\cite{eu2022dsa}, and
the AI Act imposes transparency and human-oversight duties on systems used in
democratic processes~\cite{eu2024aiact} --- and
sycophancy~\cite{perez2023sycophancy} and measurable political leaning in
pretrained models~\cite{feng2023political} are further reasons to keep such
models off the serving path.

\paragraph{Diversity, coverage and manipulation in recommendation.}
\label{sec:rw_rec}
Ranking by predicted relevance alone yields redundant lists; maximal marginal
relevance trades relevance against novelty~\cite{carbonell1998mmr}, and
aggregate diversity has been studied as a system-level
objective~\cite{adomavicius2011aggregate}. Multi-stakeholder recommendation
observes that platform, provider and consumer objectives
diverge~\cite{burke2017multisided}; recency-aware work notes the drift of
freshness against quality~\cite{chakraborty2019recency}; and conceptual
modelling of explainable recommenders has mapped what an explanation of a
recommendation can even be~\cite{caromartinez2021conceptual}. Our completeness
functional is a coverage objective in this family, with the civic-specific
twist that the covered universe is the set of \emph{reasons live at the moment
of the vote} rather than a static catalogue. The manipulation literature is
more directly load-bearing for Section~\ref{sec:adversarial}. Shilling attacks
against collaborative filtering inject profiles to promote a target
item~\cite{lam2004shilling,shardanand1995social}; Sybil attacks manufacture
identities to acquire disproportionate influence~\cite{douceur2002sybil};
eclipse attacks isolate a peer's view of the network~\cite{singh2006eclipse};
and election manipulation on social networks has been analysed as seeding and
edge modification, with the hardness of each variant
established~\cite{castiglioni2021election}. Eigenvector-style ranking
schemes~\cite{page1999pagerank} are structurally susceptible to link farms,
which is precisely why our rule takes exactly one hop and no iteration to a fixed point:
Section~\ref{sec:res_hub} confirms empirically that hub-riding is inert against
it, and Section~\ref{sec:res_flood} shows where the residual exposure lives.

\paragraph{Agentic simulation as an evaluation method.}
\label{sec:rw_agentic}
Generative agents with memory and planning reproduce plausible social behaviour
in sandboxed environments~\cite{park2023generative}; language models
conditioned on demographic profiles reproduce survey response
distributions~\cite{argyle2023outofone}; multi-agent orchestration frameworks
have matured~\cite{wu2024autogen}; and the method has been applied to
social-science questions directly~\cite{bail2024llmsocial}. Our simulator
belongs to this family and inherits its central caveat: agent behaviour is a
model of constituent behaviour, not evidence about it. We therefore restrict
every claim to statements about \emph{mechanism} under a stated behavioural
model --- comparisons between ranking arms, sensitivities to structural
parameters, responses to attacks --- and make no claim about the magnitude any
quantity would take in a human electorate.

\paragraph{Decentralised infrastructure for civic processes.}
\label{sec:rw_p2p}
Structured overlays give scalable key-based
routing~\cite{maymounkov2002kademlia}; epidemic protocols give robust eventual
dissemination~\cite{demers1987epidemic}; conflict-free replicated data types
give convergence without coordination~\cite{shapiro2011crdt}; hash trees give
compact integrity proofs~\cite{merkle1988digital}; blockchains give append-only
public ledgers~\cite{nakamoto2008bitcoin}, with transparency-log constructions
applied to update integrity~\cite{nikitin2017chainiac} and cryptographic
scrutiny of agreement protocols~\cite{miller2020cryptoag}; and decentralised
social platforms have explored gossip-based
dissemination~\cite{boutet2013whatsup}.
\DDP{} occupies a specific position: it is not a ledger and not a DHT, but a
gossip-replicated store of self-contained, signed items designed for petition
drives and organisational
decision-making~\cite{silaghi2013ddp2p,silaghi2013petition,silaghi2014encompassing}.
Its research programme covers the pieces a civic deployment actually needs ---
decentralised census construction and
verification~\cite{qin2013census,qin2014opencensus}, detection of false
identities~\cite{qin2013falseidentity}, peer
reputation~\cite{qin2013reputation}, supernodes for peers behind
NAT~\cite{alhamed2014supernode}, protocol stacking and update
propagation~\cite{alhamed2013stacking,alhamed2013updates}, meta-level
recommendation over decentralised
data~\cite{alhamed2016metarecommenders}, trust and key
management~\cite{silaghi2016pgp}, interface studies for non-expert
users~\cite{alqahtani2017hci,alqahtani2016cognition,kattamuri2005supporting},
its logical foundations~\cite{roussev2017logic}, related vehicular
deployments~\cite{dhannoon2013vanet,dhannoon2013thesis}, and the motivating
argument for why participation is worth the
engineering~\cite{silaghi2017whyvote}. Section~\ref{sec:p2p} builds on that
stack rather than proposing a new one.
\section{The Object: An Alternative-Based Poll and Three Ways to Judge a Slate}
\label{sec:model}

Before anything can be recommended, the thing being recommended over has to be
defined, and before any recommender can be judged, the standard of judgement
has to be fixed. This section does both. The object is an
\emph{alternative-based information system}: a poll in which each side of the
question carries its own set of justifications, and in which typed links
between justifications carry the polarity of the relations participants assert
between them. The standard is deliberately plural. We define three
instruments --- how much of the live reason vocabulary a slate covers, how
early in the slate that coverage arrives, and how much of the electorate's
expressed endorsement the slate captures --- because Section~\ref{sec:instruments}
will show that any one of them alone gives a misleading verdict. We then state
the combinatorial problem that coverage induces, prove it hard, and construct
the greedy bound we use as an evaluation ceiling, with an explicit statement
of the discipline governing that bound's use.

\subsection{Polls, sides, and justifications}
\label{sec:polls}

A poll asks a question with a finite set of mutually exclusive alternatives.
Throughout we take the binary case --- support or oppose a motion --- because
it is the case civic petitions actually present and because it keeps the
notation legible; nothing in the construction depends on it, and
Section~\ref{sec:conclusion} notes the multi-alternative generalisation.

\begin{definition}[Alternative-based poll]
\label{def:poll}
An \emph{alternative-based poll} is a tuple
\begin{equation}
  \Pi \;=\; \bigl\langle\, \Jpos,\ \Jneg,\ \Rreb,\ \Rrei,\ \wc,\ \wrel \,\bigr\rangle
  \label{eq:poll}
\end{equation}
where $\Jpos$ and $\Jneg$ are finite sets of \emph{justifications}
attached to the supporting and opposing alternatives respectively\footnote{In the experiments reported here, $\Jpos$ and $\Jneg$ are disjoint, but that can be relaxed.};
$\Rreb \subseteq \mathcal{J} \times \mathcal{J}$ with
$\mathcal{J} = \Jpos \cup \Jneg$ is the \emph{rebuttal} relation, read
``$(j,k) \in \Rreb$ asserts that $j$ tells against $k$'';
$\Rrei \subseteq \mathcal{J} \times \mathcal{J}$ is the \emph{reinforcement}
relation, read ``$j$ tells in favour of $k$''; $\wc : \mathcal{J} \to
\mathbb{R}_{\ge 0}$ assigns each justification a civic weight; and
$\wrel : \Rreb \cup \Rrei \to \mathbb{R}_{\ge 0}$ assigns each link a weight.
\end{definition}

Three features carry the design. The two relations are typed but not signed at
the level of the tuple: polarity lives in which relation a link belongs to, and
the score function of Section~\ref{sec:rule} attaches a separate coefficient to
each, so a participant may configure how much a rebuttal counts relative to a
reinforcement --- a substantive editorial choice --- without reaching inside
the graph. The civic weight $\wc$ is the endorsement count: how many
constituents cast their ballot citing that item. It is a public tally, not a
quality judgement, and no part of the construction asks whether an item
deserves the endorsements it has. Finally $\wrel$ is where policy lives:
Section~\ref{sec:weights} shows that choosing it flat or normalised by the
number of distinct authors behind the link moves completeness under coordinated
flooding by $0.08$--$0.12$, an order of magnitude larger than any other design
decision here. 

\subsection{Reason vocabularies and the live denominator}
\label{sec:tags}

Coverage requires a universe to cover. We obtain it by labelling each
justification with the reasons it invokes.

\begin{definition}[Reason labelling]
\label{def:labels}
Fix a finite \emph{reason vocabulary} $\Lam$. A \emph{labelling} is a map
$\lab : \mathcal{J} \to 2^{\Lam}$ assigning each
justification the non-empty set of reasons it invokes\footnote{In the first set of reported experiments, $\lab : \mathcal{J} \to 2^{\Lam} \setminus \{\emptyset\}$ to have solely non-degenerate authorship.}.
\end{definition}

Just for evaluation in reported experiments, labels are associated at authoring time (Section~\ref{sec:rw_mining}), stored on
the item, and never recomputed at serving time; a deployment publishes its
vocabulary and its extraction procedure.
The denominator is the
subtler half. Comparing a slate against the full vocabulary $\Lam$ would
penalise an early voter's slate for failing to see reasons nobody had yet written,
which is exactly the temporal inequity of Section~\ref{sec:exposure} --- a real
phenomenon, but one to \emph{measure} separately rather than fold into the
score of the recommender. We therefore evaluate against what existed.

\begin{definition}[Live vocabulary]
\label{def:live}
For a voter $i$ casting a ballot at time $t_i$, let
$\mathcal{J}_{t_i}$ be the set of justifications submitted strictly before
$t_i$, and let the \emph{live vocabulary} on side $\sigma \in \{+,-\}$ be
$\Lam^{\sigma}_{t_i} = \bigcup_{\,j \in \mathcal{J}^{\sigma} \cap \mathcal{J}_{t_i}} \lab(j)$.
\end{definition}

\begin{definition}[Slate completeness]
\label{def:completeness}
Let $S_i^{\sigma} \subseteq \mathcal{J}^{\sigma} \cap \mathcal{J}_{t_i}$ be
the slate of at most $K$ justifications served to voter $i$ on side $\sigma$.
Its \emph{completeness} is
\begin{equation}
  c_i^{\sigma} \;=\;
  \frac{\bigl|\ \bigcup_{j \in S_i^{\sigma}} \lab(j)\ \bigr|}
       {\bigl|\ \Lam^{\sigma}_{t_i}\ \bigr|}
  \;\in\; [0,1],
  \label{eq:completeness}
\end{equation}
with $c_i^{\sigma} = 1$ by convention when $\Lam^{\sigma}_{t_i} = \emptyset$.
The \emph{combined} completeness of a run over $N$ voters, by averaging, is
\begin{equation}
  \bar{c} \;=\; \frac{1}{2N}\sum_{i=1}^{N}\bigl(c_i^{+} + c_i^{-}\bigr).
  \label{eq:meancompleteness}
\end{equation}
\end{definition}

Completeness is defined per voter and per side, and \eqref{eq:meancompleteness}
aggregates a voter's two values by averaging them. Averaging permits
compensation: a slate that covers one side well and the other badly scores the
same as one that covers both indifferently. A flooding coalition attacks one
side, so this is precisely the configuration in which the reported number and
the experienced one can come apart. 
We therefore also run evaluations with the average over a voter's two sides replaced by the minimum,
\begin{equation}
  c_i^{\min} \;=\; \min\bigl(c_i^{+},\, c_i^{-}\bigr),
  \qquad
  \bar{c}^{\min} \;=\; \frac{1}{N}\sum_{i=1}^{N} c_i^{\min},
  \label{eq:mincompleteness}
\end{equation}
under which a voter is served exactly as well as their worse-covered side and
is granted no credit for coverage they did not lose.

Two properties of \eqref{eq:completeness} deserve to be stated here rather
than discovered later, because between them they account for the entire
narrative arc of Section~\ref{sec:instruments}.

\begin{remark}[Order blindness]
\label{rem:orderblind}
$c_i^{\sigma}$ depends on $S_i^{\sigma}$ only as a set. Permuting the slate
leaves it unchanged. Completeness therefore cannot, even in principle,
distinguish a ranking rule from any procedure that returns the same $K$ items
in a different order --- and when $|\mathcal{J}^{\sigma} \cap \mathcal{J}_{t_i}|$
is not much larger than $K$, it cannot distinguish it from a random draw
either, because both return nearly the same set.
\end{remark}

\begin{remark}[Charity blindness]
\label{rem:charityblind}
$c_i^{\sigma}$ credits an item for the labels it carries, assumed to be a nonempty set of a size above some given positive threshold, whatever the quality
of the argument attached to them. If every submission in the corpus is a
well-formed argument on its own side, every item is worth roughly the same to
this measure, and no selection rule can beat a random draw by much. The
measure only rewards discrimination when there is something to discriminate
\emph{against}.
\end{remark}

\subsection{Three instruments: coverage, order, and endorsement mass}
\label{sec:instruments_def}

A slate is a sequence, presented to a person with finite attention, drawn from
a corpus the electorate has already expressed opinions about. Each clause
supports a different question about whether the slate was well chosen, and we
report all three because Section~\ref{sec:instruments} shows they disagree in
ways that matter.

\paragraph{Instrument I: coverage.} Completeness $\bar{c}$ per
Definition~\ref{def:completeness}. It answers: \emph{were the reasons in play
represented at all?} It is the natural formalisation of the balance
requirement inherited from deliberative polling, and it is order-blind and
charity-blind per Remarks~\ref{rem:orderblind}--\ref{rem:charityblind}.

\paragraph{Instrument II: order.} Real readers do not consume slates
uniformly. Let $S_i^{\sigma} = \langle j_1, \dots, j_K \rangle$ now be an
ordered slate, and define the \emph{prefix coverage} at depth $m$,
\begin{equation}
  p^{\sigma}_i(m) \;=\;
  \frac{\bigl|\ \bigcup_{u \le m} \lab(j_u)\ \bigr|}{\bigl|\ \Lam^{\sigma}_{t_i}\ \bigr|},
  \qquad m = 1,\dots,K,
  \label{eq:prefix}
\end{equation}
so that $p^{\sigma}_i(K) = c^{\sigma}_i$ recovers Instrument~I. From the
prefix profile we take three summaries. The \emph{area under the prefix curve}
$\AUC^{\sigma}_i = K^{-1}\sum_{m=1}^{K} p^{\sigma}_i(m)$
is the coverage enjoyed by a reader who stops at a uniformly random depth. The
\emph{rank-discounted coverage} weights each reason by where it first appears,
\begin{equation}
  \RDC^{\sigma}_i \;=\;
  \Bigl(\textstyle\sum_{t \le |\Lam^{\sigma}_{t_i}|} \tfrac{1}{\log_2(1+t)}\Bigr)^{-1}
  \sum_{\ell \in \Lam^{\sigma}_{t_i}} \frac{\mathbb{1}[\ell \text{ served}]}{\log_2\bigl(1 + \mathrm{pos}_i^{\sigma}(\ell)\bigr)},
  \label{eq:rdc}
\end{equation}
where $\mathbb{1}[\ell \text{ served}]$ is an indicator function evaluating to $1$ if the reason $\ell$ appears on at least one item of the slate served to voter $i$ on side $\sigma$, and $0$
  otherwise, $\mathrm{pos}_i^{\sigma}(\ell)$ is the position of the first slate item
carrying $\ell$ and the normaliser is the value a slate would attain by
covering the whole live vocabulary as early as positions allow, the normalizer being replaced with 1 at the beginning when that is needed to avoid division by 0. Finally
$e_{90}$ is the smallest level $m$ with $p^{\sigma}_i(m) \ge 0.9\,c^{\sigma}_i$ ---
how deep the reader must go to obtain nine tenths of what that slate was ever
going to give them. Instrument~II answers: \emph{did the coverage arrive early
enough to be read?}

\paragraph{Instrument III: endorsement mass.} A slate can cover the vocabulary
while serving nothing that any constituent actually adopted. Define
\begin{equation}
  \mass^{\sigma}_i \;=\;
  \frac{\sum_{j \in S^{\sigma}_i} \wc(j)}
       {\max_{\,T \subseteq \mathcal{J}^{\sigma} \cap \mathcal{J}_{t_i},\ |T| \le K}\ \sum_{j \in T} \wc(j)},
  \label{eq:mass}
\end{equation}
the share of the maximum achievable endorsement weight that the served slate
captured. Instrument~III answers: \emph{did the slate show the voter what the
electorate had actually taken up?} It is the axis on which
Section~\ref{sec:res_pareto} finds the random baseline dominated by a factor of
$3.3$, and the axis on which the greedy oracle of Section~\ref{sec:oracle}
stops dominating under attack. Nothing in the design privileges one instrument:
Section~\ref{sec:res_pareto} treats them jointly, as a frontier over which a
participant --- not the system --- chooses a position, and that this is even
expressible is a consequence of Criterion~\ref{c:config} (configurability,
Section~\ref{sec:desiderata}).

\subsection{The Subsuming Justification Problem}
\label{sec:sjp}

Instrument~I induces a combinatorial problem which we name and characterise,
both because it locates the construction in complexity terms and because its
hardness is what licenses the greedy bound of the next subsection. Say that
$S \subseteq \mathcal{J}^{\sigma}$ \emph{subsumes} a target vocabulary
$L \subseteq \Lam$ when $\bigcup_{j \in S} \lab(j) \supseteq L$.

\begin{definition}[Subsuming Justification Problem, SJP]
\label{def:sjp}
\emph{Instance:} a labelled justification set $(\mathcal{J}^{\sigma}, \lab)$,
a target $L \subseteq \Lam^{\sigma}$, and a budget $K \in \mathbb{N}$.
\emph{Question:} does some $S \subseteq \mathcal{J}^{\sigma}$ with $|S| \le K$
subsume $L$?
\end{definition}

\begin{proposition}
\label{prop:np}
SJP is \textsf{NP}-complete.
\end{proposition}

\begin{proof}
Membership: a certificate $S$ with $|S| \le K$ is verified by computing
$\bigcup_{j \in S}\lab(j)$ and testing containment of $L$, in time linear in
$\sum_{j\in S}|\lab(j)|$. Hardness: reduce from \textsc{Set
Cover}~\cite{karp1972reducibility}. Given a universe $U$, a family
$\mathcal{F} = \{F_1,\dots,F_n\}$ of subsets of $U$, and a budget $K$, put
$\Lam = U$, create one justification $j_r$ per $F_r$ with $\lab(j_r) = F_r$,
set $L = U$, and keep the budget. Any $S$ of size $\le K$ subsuming $U$ maps
to a cover of size $\le K$ and conversely; the construction is linear in the
instance size.
\end{proof}

Two variants matter operationally. \emph{Weighted SJP} adds
$\mu : \Lam \to \mathbb{R}_{\ge 0}$ and a threshold $\tau$ and asks whether
some $S$ with $|S| \le K$ attains
$\mu(\bigcup_{j\in S}\lab(j)) \ge \tau$; it inherits hardness at
$\mu \equiv 1$, $\tau = |L|$, and its interest is normative rather than
computational, since $\mu$ is where a deployment would encode that some reasons
must not be crowded out --- minority-held reasons, statutory considerations ---
exactly the kind of parameter that must be published in advance rather than
learned, a learned $\mu$ being an unaccountable editorial line.
\emph{Componentwise SJP} reflects that a slate covering one side well and the
other badly is not balanced: given budgets $K^{+}, K^{-}$ and targets
$L^{+}, L^{-}$, it asks for $S^{+}, S^{-}$ with $|S^{\sigma}| \le K^{\sigma}$
subsuming $L^{\sigma}$ for both $\sigma$. Since the two sides share no items it
decomposes into two independent SJP instances --- which is why
\eqref{eq:meancompleteness} averages per-side completeness rather than pooling
the vocabularies: pooling would let a slate compensate a neglected side with an
over-covered one, the opposite of balance.

\subsection{The greedy ceiling, and the discipline governing its use}
\label{sec:oracle}

Since coverage is the objective of an \textsf{NP}-hard problem, an exact
optimum is not an available comparator. A standard bound is, and it happens to be
tight enough to be informative. Given
$(\mathcal{J}^{\sigma}\cap\mathcal{J}_{t_i}, \lab)$, target
$\Lam^{\sigma}_{t_i}$ and budget $K$, the \emph{greedy cover} selects, $K$
times, the item maximising the number of not-yet-covered reasons it
contributes, breaking ties by a published deterministic key.

\begin{proposition}
\label{prop:greedy}
The reason count $f(S) = |\bigcup_{j\in S}\lab(j)|$ is monotone and
submodular, so the greedy cover attains at least $(1 - e^{-1}) \approx 0.632$
of the optimal value at budget $K$~\cite{nemhauser1978analysis}, and no
polynomial-time algorithm improves the ratio unless
$\textsf{P} = \textsf{NP}$~\cite{feige1998threshold,hochbaum1997approximating}.
\end{proposition}

\begin{proof}
Monotonicity is immediate. For submodularity, take $S \subseteq T$ and
$j \notin T$; then
$f(S\cup\{j\}) - f(S) = |\lab(j) \setminus \bigcup_{k\in S}\lab(k)|
 \ge |\lab(j)\setminus\bigcup_{k\in T}\lab(k)| = f(T\cup\{j\}) - f(T)$,
since the subtracted set is larger. The bound and its optimality follow.
\end{proof}

The greedy cover reads $\lab$ directly --- a semantic operation requiring the
selector to know what each item is about and to choose items for what they say.
It therefore violates Criterion~\ref{c:abstain} (semantic abstinence) outright,
and because it must inspect every candidate's labels it also fails
Criterion~\ref{c:local} (evidence locality); both are stated with their tests in
Section~\ref{sec:desiderata}. It is
not a candidate mechanism, yet we compare against it constantly, which requires
stating the discipline explicitly.

\begin{principle}[Oracle discipline]
\label{prin:oracle}
Like the labels, the greedy cover appears in this manuscript only on the evaluation path. It is
computed offline, from the persisted database, after the poll has closed. No
agent ever sees a slate it produced; no served slate was ever influenced by
it; it is not proposed, and could not be adopted, as a deployable selection
rule. Its role is to answer one question that no admissible mechanism can
answer about itself: \emph{how much coverage was available in this pool at all?}
\end{principle}

Principle~\ref{prin:oracle} is what makes the number in
Section~\ref{sec:res_ceiling} meaningful. Because the greedy cover reads the
labels, its coverage upper-bounds --- to within the $0.632$ factor, and in
practice far more tightly --- what \emph{any} selection procedure applied to
the same pool could have achieved, including an arbitrarily good opaque learned
ranker. The gap between the served slates and that ceiling is therefore not the
price of our particular rule against some better rule; it is the price of the
entire admissible class against the unreachable best case, and
Section~\ref{sec:res_ceiling} measures it at $0.031 \pm 0.014$.

\subsection{Reach: what a slate can be held responsible for}
\label{sec:reach}

A rule that reads only endorsements and one-hop links cannot serve what it
cannot see. Making that horizon explicit turns a limitation into a
specification. The \emph{reach} of a justification $j$ at time $t$ is
$\rho_t(j) = \{j\} \cup \{\,k : (j,k) \in \Rreb \cup \Rrei,\ k \in
\mathcal{J}_t \,\}$, and $j$ is \emph{visible to the rule} at $t$ if
$\wc(j) > 0$ or $\wc(k) > 0$ for some $k \in \rho_t(j)$.

\begin{proposition}[Partial-replica sufficiency]
\label{prop:partial}
Evaluating the endorsement rule of Section~\ref{sec:rule} for the items in a
candidate pool $P$ requires only the endorsement counts of
$\bigcup_{j\in P}\rho_t(j)$ and the link tuples out of $P$. It requires no
global view of $\mathcal{J}$, $\Rreb$ or $\Rrei$. This is immediate from the
score \eqref{eq:score} given in Section~\ref{sec:rule}: the score of $j$ is a
function of $\wc(j)$ and of $\wc(k)$
for $k$ ranging over $j$'s out-neighbours only.
\end{proposition}

Proposition~\ref{prop:partial} is the technical fact that makes
Section~\ref{sec:p2p} possible. A peer holding a partial replica computes
exactly the same scores for the items it holds as a peer holding everything,
which is what allows every participant to run the rule themselves rather than
trust an operator to run it for them. It also delimits responsibility:
a brand-new item with no endorsements and no incoming links from endorsed
material is invisible to the rule, and no amount of tuning changes that.
Cold-start items surface when someone reads them --- which is what the
authoring-side random exploration of Section~\ref{sec:round} is for --- and not
before.
\subsection{A worked example}
\label{sec:example}

The following instance is small enough to check by hand and exhibits every
phenomenon the experiments will measure at scale.

\begin{example}[Municipal night-bus extension]
\label{ex:bus}
A city puts to its residents: \emph{extend the night-bus network to the outer
districts?} At the moment resident $r$ opens the ballot, the supporting side
holds five justifications and the opposing side four. Labels are drawn from a
published vocabulary; endorsement counts $\wc$ are the public tallies.

\begin{center}\small
\begin{tabular}{@{}llcl@{}}
\toprule
\textbf{id} & \textbf{gist} & $\wc$ & \textbf{labels } $\lab$ \\
\midrule
\multicolumn{4}{@{}l}{\textit{supporting side} $\Jpos$}\\
$s_1$ & shift workers stranded after 23:00 & 46 & \textsc{access}, \textsc{equity} \\
$s_2$ & night transit cuts drink-driving crashes & 31 & \textsc{safety} \\
$s_3$ & same rolling stock, marginal cost only & 12 & \textsc{cost} \\
$s_4$ & shift workers are mostly low-income & \phantom{0}4 & \textsc{equity} \\
$s_5$ & fewer night cars lowers emissions & \phantom{0}3 & \textsc{climate} \\
\midrule
\multicolumn{4}{@{}l}{\textit{opposing side} $\Jneg$}\\
$o_1$ & subsidy competes with school repairs & 39 & \textsc{cost}, \textsc{priorities} \\
$o_2$ & measured outer-district demand is very low & 28 & \textsc{demand} \\
$o_3$ & depot noise at 02:00 in residential streets & \phantom{0}9 & \textsc{amenity} \\
$o_4$ & rosters would breach the rest agreement & \phantom{0}2 & \textsc{labour} \\
\bottomrule
\end{tabular}
\end{center}

Both live vocabularies have five reasons, so $|\Lam^{+}| = |\Lam^{-}| = 5$.
The links asserted so far are $(s_3, o_1) \in \Rreb$ (marginal cost tells
against the subsidy objection), $(s_4, s_1) \in \Rrei$, $(s_2, o_3) \in \Rreb$,
$(o_2, s_1) \in \Rreb$ and $(o_4, o_1) \in \Rrei$. The display budget is
$K = 2$ per side.

\medskip\noindent\textbf{What each procedure serves.}
Take the rule of Section~\ref{sec:rule} at the reference policy
$\alpha = \beta = 0.5$, $\wrel \equiv 1$. On the supporting side,
$\mathrm{sc}(s_1) = 46$, $\mathrm{sc}(s_2) = 31 + 0.5\cdot 9 = 35.5$,
$\mathrm{sc}(s_3) = 12 + 0.5\cdot 39 = 31.5$,
$\mathrm{sc}(s_4) = 4 + 0.5\cdot 46 = 27$, $\mathrm{sc}(s_5) = 3$.
The rule serves $\langle s_1, s_2 \rangle$, covering
$\{\textsc{access},\textsc{equity},\textsc{safety}\}$: completeness $3/5 = 0.60$,
prefix profile $p(1) = 0.40$, $p(2) = 0.60$. The greedy cover reads labels and
serves $\{s_1, s_3\}$ or $\{s_1, s_5\}$ --- also $0.60$, since no pair here
exceeds three reasons. A uniform draw serves $2.4$ distinct reasons in
expectation, i.e.\ completeness $0.48$.

\medskip\noindent\textbf{Where the instruments diverge.}
Let the draw return $\{s_1, s_5\}$, which it does as often as any other pair:
its coverage is $3/5$, equal to the rule's, so on Instrument~I the two
procedures are indistinguishable on this run. On Instrument~III they are not:
the rule captures $\mass^{+} = 77/77 = 1.00$ against $49/77 = 0.64$ for
$\{s_1,s_5\}$ and $15/77 = 0.19$ for $\{s_3,s_5\}$. The voter shown
$\{s_3, s_5\}$ met two reasons four of their neighbours had ever endorsed;
the voter shown $\langle s_1, s_2\rangle$ met the two that seventy-seven had.

\medskip\noindent\textbf{Where the link terms earn their place.}
Item $s_3$ ranks fourth on its own endorsements and is promoted to third by
the rebuttal coefficient because it speaks to the most endorsed item on the
\emph{other} side --- but ranking on endorsements alone would place it third
as well. That is the inertia of Section~\ref{sec:res_null} in miniature: on a
uniformly competent corpus the link terms reorder items that were already
adjacent. Now replace $s_5$ by a submission $s_5'$ carrying the label
\textsc{climate} and no actual reason: a sentence of agreement, correctly
labelled, endorsed by nobody, asserting no links. Coverage is blind to the
substitution, so a uniform draw serves $s_5'$ exactly as often as it served
$s_5$; the rule scores it $0$, both because nobody endorsed it and because it
points at nothing anyone endorsed. That double penalty --- the \emph{squaring}
of discrimination measured in Section~\ref{sec:res_mechanism} --- is invisible
here only because the instance contains one such item. At the fractions
Section~\ref{sec:res_degen} reports it is worth as much as the whole ranking
advantage.
\end{example}

\section{The Charter: Criteria Before Objectives}
\label{sec:charter}

This section states what a selection mechanism must satisfy to be used in a
binding civic process, and exhibits a rule that satisfies it. The order of
presentation within this section is itself part of the argument: the criteria
come before the rule, as conditions on the admissible class, and the rule
follows as one member of that class chosen for its simplicity. 
We begin with what opacity costs in concrete operational terms, state
the seven criteria with their tests, argue that they are prior to rather than
commensurable with measured quality, give the rule, isolate the one line of it
that is a security control, and close on how per-voter configuration is
possible without dissolving the shared record.

\subsection{What opacity costs, concretely}
\label{sec:opacity}

The case against an opaque selector here is not that such systems are
inaccurate; it is that some operations a civic process requires become
impossible, and it is worth naming them as operations rather than as values.

\emph{Recomputation}: a participant who disputes a slate should be able to
obtain the same slate from the same inputs, which with published weights over
public evidence is arithmetic, whereas with a learned selector the participant
would need the model, its parameters, the exact feature vector at serving time
and the personalisation state --- and even a cooperative operator disclosing all
four supplies a recomputation nobody outside can independently verify was the
one actually run. 

\emph{Attribution}: when a slate is wrong, a process needs to
say which input made it wrong, and under the rule of Section~\ref{sec:rule} the
answer is a term --- this item ranked here because it carried these endorsements
and pointed at those items --- whereas under a learned ranker the answer
is that the output is a function of the training distribution, post-hoc
attribution methods producing explanations that are themselves
unverifiable~\cite{rudin2019stop,lipton2018mythos}; surveys of explainability
for supervised learning make the gap between an explanation and a derivation
plain~\cite{burkart2021survey}. 

\emph{Explanation}: a voter is owed an account of why \emph{their} slate contains what it contains, and that account must be the computation rather than a story fitted to it. Under the rule it is the score read aloud
and cannot drift from the mechanism because it is the mechanism, whereas a learned selector explains itself through a second model whose account is plausible to the recipient rather than identical to the computation that ran~\cite{caromartinez2021conceptual}.
An explanation a voter cannot check against the public record is indistinguishable from persuasion.

\emph{Contest}: a dispute must terminate, and
under the rule it terminates in one of three ways --- an endorsement tally is
wrong and tallies are public, a link is wrong and links are signed by their
authors, or the policy is wrong and that is a political argument conducted in
the open --- whereas under an opaque selector the dispute has no terminating
move, which converts every disagreement about content into a standing grievance
about the operator. \

emph{Bounded drift}: a published rule changes when someone
changes it and the change is a diff, whereas a learned selector changes
whenever it is retrained, whenever the population shifts and whenever a feature
pipeline is updated, so its behaviour last month is not recoverable --- and a
civic record that cannot be reconstructed a year later is not a record.

None of these costs is offset by better ranking, because none is a ranking
property. This is the structural reason the criteria below are stated as
admissibility conditions rather than as objectives to be weighed, and it is
also where this construction parts company with the fairness-and-accountability
literature that treats such properties as quantities to optimise: proposals for
socially responsible AI~\cite{cheng2021socially} and critiques of fairness as a
single formal target~\cite{weinberg2022rethinking} both proceed by asking what
a system should maximise, whereas the question here is which systems may be
used at all.

\subsection{Seven criteria, with their tests}
\label{sec:desiderata}

Each criterion below is stated so that an auditor can check it against an
implementation, together with the observable that would show it violated. They
are numbered for reference throughout the rest of this manuscript; every
subsequent design decision is annotated with the criterion it discharges.

\begin{criterion}[Determinism]
\label{c:det}
Given the same evidence and the same policy vector, the mechanism returns the
same slate. \emph{Test:} run it twice on a frozen snapshot; the outputs are
identical, ties included. \emph{Violation looks like:} two participants with
identical configurations and identical local data seeing different slates,
with no published reason.
\end{criterion}

\begin{criterion}[Evidence locality]
\label{c:local}
The score of an item depends only on that item and on a bounded, explicitly
specified neighbourhood of it. \emph{Test:} perturb an item outside the
declared neighbourhood; the score does not move.
\emph{Violation looks like:} a global fixed point, or any quantity whose value
depends on the whole corpus, which makes both partial-replica evaluation and
local reasoning about a dispute impossible.
\end{criterion}

\begin{criterion}[Author blindness]
\label{c:author}
No term in the score refers to who wrote an item, beyond counting
\emph{distinct} authors where a policy explicitly calls for it. \emph{Test:}
permute author identities across items; the ranking is invariant.
\emph{Violation looks like:} reputation, seniority, or verified-account status
entering the ranking --- reintroducing the influence hierarchy that a poll is
supposed to flatten.
\end{criterion}

\begin{criterion}[Semantic abstinence]
\label{c:abstain}
The serving-time mechanism does not read the content of items. \emph{Test:}
replace every item's text with an opaque identifier; the slate is unchanged.
\emph{Violation looks like:} the selector deciding, at serving time, what an
argument means or whether it is any good --- which is the specific authority a
civic slate must not delegate, and which the greedy cover of
Section~\ref{sec:oracle} openly exercises, hence Principle~\ref{prin:oracle}.
\end{criterion}

\begin{criterion}[Reproducibility]
\label{c:repro}
Every served slate is recomputable after the fact from persisted evidence.
\emph{Test:} from the archive, reconstruct any historical slate exactly,
including the state of the corpus at that instant. \emph{Violation looks
like:} an audit that can establish what was tallied but not what was shown.
\end{criterion}

\begin{criterion}[Contestability]
\label{c:contest}
Every slate decomposes into named contributions, each traceable to a public
artefact. \emph{Test:} for any served item, produce the list of
$(\text{term}, \text{evidence}, \text{value})$ triples summing to its score.
\emph{Violation looks like:} an explanation that is itself a model output.
\end{criterion}

\begin{criterion}[Configurability]
\label{c:config}
The policy parameters are held by the individual participant receiving it, not by the
operator, and the participant can change them and observe the effect.
\emph{Test:} two participants with different policies, same evidence, obtain
different and separately correct slates. \emph{Violation looks like:} a single
operator-chosen ranking presented as neutral, or a personalisation the
participant cannot inspect or switch off.
\end{criterion}

Table~\ref{tab:charter} summarises how three mechanism families fare. The
pattern is not that learned rankers score lower; it is that they are
disqualified on four criteria at once, by construction rather than by
implementation quality.

\begin{table}[tb]
\centering\footnotesize
\caption{The seven criteria against three mechanism families.
\textbf{--} denotes failure by construction rather than by implementation
choice: no amount of engineering effort within that family recovers the
property. The greedy cover is included because it appears throughout the
evaluation and it is important to be explicit that it is inadmissible;
see Principle~\ref{prin:oracle}.}
\label{tab:charter}
\begin{tabular}{@{}p{4.3cm}ccc@{}}
\toprule
& \textbf{endorsement rule} & \textbf{greedy cover} & \textbf{learned ranker} \\
& (\S\ref{sec:rule}) & (\S\ref{sec:oracle}) & \\
\midrule
C\ref{c:det} determinism        & yes & yes & only if frozen \\
C\ref{c:local} evidence locality & one hop & whole pool \textbf{--} & unbounded \textbf{--} \\
C\ref{c:author} author blindness & yes & tie-breaks & learned from behaviour \textbf{--} \\
C\ref{c:abstain} semantic abstinence & yes & reads labels \textbf{--} & reads everything \textbf{--} \\
C\ref{c:repro} reproducibility   & yes & yes & needs exact checkpoint \textbf{--} \\
C\ref{c:contest} contestability  & per-term & per-item & post-hoc only \textbf{--} \\
C\ref{c:config} configurability  & per voter & fixed & operator-held \textbf{--} \\
\bottomrule
\end{tabular}
\end{table}

\subsection{Why this is not an empirical hypothesis}
\label{sec:notempirical}

A natural objection runs: \emph{the assertion that rule-based selection is
preferable, is never tested against a learned alternative.
} The objection is well formed and the answer
governs how the rest of the results should be read. The claim is normative and
it is a claim about admissibility: a mechanism that cannot be recomputed,
attributed, explained, contested or reconstructed is unsuitable for a binding civic
process --- not that it ranks worse. A comparison against a learned ranker
would report a difference in coverage or in some downstream engagement
statistic, and whatever number came out could not bear on the claim, since a
learned ranker covering \emph{more} of the live vocabulary would still fail
C\ref{c:local}, C\ref{c:abstain}, C\ref{c:repro} and C\ref{c:contest} and would
still leave a disputing participant with no terminating move. 
The structure
is familiar from other civic procedures. The secret ballot is not defended on
the grounds that it measures preferences more accurately than open voting --- it
plainly measures some things less well, since it destroys the ability to audit
an individual's vote --- but because the procedure must possess a property that
outranks measurement quality. Double-entry bookkeeping is not the most compact
representation of a firm's accounts; rules of order do not produce the fastest
decisions. In each case a procedural property is treated as prior, and
efficiency questions are settled \emph{within} the class of procedures that
possess it.

What is legitimately empirical, and what this manuscript tests, is everything
downstream of that commitment: \emph{how much does the commitment cost?},
answered against a ceiling that upper-bounds every mechanism including
inadmissible ones (Section~\ref{sec:res_ceiling}); \emph{which terms of the
rule earn their place?}, answered by ablation, including the finding that two
of them earn nothing on a charitable corpus and a great deal on a realistic one
(Sections~\ref{sec:res_null} and~\ref{sec:res_degen}); \emph{which policy
choices are security controls?}, answered by adversarial sweep
(Section~\ref{sec:res_flood}); and \emph{where does the choice between
configurations actually lie?}, answered by a frontier rather than an optimum
(Section~\ref{sec:res_pareto}). 

\subsection{The endorsement rule}
\label{sec:rule}

We now give a mechanism satisfying all seven criteria.
Its simplicity is a feature.

\begin{definition}[Endorsement rule]
\label{def:rule}
Let $E(j) = \wc(j)$ be the endorsement count of item $j$. Given a policy
vector $\theta = (\alpha, \beta, \wrel)$ with $\alpha, \beta \in \mathbb{R}$,
the \emph{endorsement score} of $j$ at time $t$ is
\begin{equation}
  \mathrm{sc}_{\theta}(j)
  \;=\;
  E(j)
  \;+\;
  \alpha \!\!\sum_{(j,k) \in \Rrei,\ k \in \mathcal{J}_t}\!\! \wrel(j,k)\, E(k)
  \;+\;
  \beta \!\!\sum_{(j,k) \in \Rreb,\ k \in \mathcal{J}_t}\!\! \wrel(j,k)\, E(k).
  \label{eq:score}
\end{equation}
The slate served to voter $i$ on side $\sigma$ is the top-$K$ of
$\mathcal{J}^{\sigma} \cap \mathcal{J}_{t_i}$ by $\mathrm{sc}_{\theta_i}$,
ties broken by a published deterministic key (ascending item identifier).
\end{definition}

The direction of both summands is the load-bearing detail and is easy to get
backwards. An item is credited for the endorsements of what it \emph{points
at}, not for links pointing at it. An item that rebuts a heavily endorsed
objection is thereby promoted, because a voter about to accept that objection
has a specific interest in seeing the response to it; and an item reinforcing a
heavily endorsed claim is promoted because it deepens material the electorate
has already taken up. Reversing the direction would reward being talked about,
which is the property a coordinated group can manufacture most cheaply --- and
it is why the eigenvector family, which propagates credit backwards along links
to a fixed point~\cite{page1999pagerank}, is structurally exposed to link farms
in a way Equation~\eqref{eq:score} is not. Section~\ref{sec:res_hub} confirms this: a
coalition that constructs a hub and points the whole corpus at it gains nothing
measurable.

The rule discharges the charter as follows. It is a closed-form arithmetic
expression over integers and published constants, hence C\ref{c:det}. It reads
$j$ and $j$'s out-neighbours and nothing else, hence C\ref{c:local}, which by
Proposition~\ref{prop:partial} is also what makes per-peer evaluation more possible.
No term names an author, and where a policy counts authors it counts
\emph{distinct} ones without regard to which, hence C\ref{c:author}. No term
reads text, hence C\ref{c:abstain} --- note that $\lab$ appears nowhere in
\eqref{eq:score}; labels are used to \emph{evaluate} slates and never to choose
them, exactly the asymmetry Principle~\ref{prin:oracle} protects. Persisting
endorsements and links with timestamps makes every historical slate
recomputable, hence C\ref{c:repro}. And each of the three summands is
separately displayable against the artefacts that produced it, hence
C\ref{c:contest}. C\ref{c:config} is the subject of Section~\ref{sec:config}.
Algorithm~\ref{alg:rank} states the serving procedure for a peer holding a partial replica.

\begin{algorithm}[tb]
\caption{\textnormal{\textsc{ScoreAndServe}} --- serve a slate for one voter on one side}
\label{alg:rank}
\SetAlgoLined
\Given{candidate pool $P \subseteq \mathcal{J}^{\sigma} \cap \mathcal{J}_t$;
       endorsement counts $E$ on $P \cup \bigcup_{j \in P}\rho_t(j)$;
       out-links of $P$; voter policy $\theta = (\alpha,\beta,\wrel)$;
       budget $K$}
\Yields{ordered slate $\langle j_1,\dots,j_{\min(K,|P|)}\rangle$ and,
        for each, its score decomposition}
\ForEach{$j \in P$}{
  $a \leftarrow \sum_{(j,k) \in \Rrei,\, k \in \mathcal{J}_t} \wrel(j,k)\,E(k)$\;
  $r \leftarrow \sum_{(j,k) \in \Rreb,\, k \in \mathcal{J}_t} \wrel(j,k)\,E(k)$\;
  $\mathrm{sc}[j] \leftarrow E(j) + \alpha a + \beta r$\;
  $\mathrm{why}[j] \leftarrow \bigl\langle (\textsf{own}, E(j)),\ (\textsf{reinforces}, \alpha a),\ (\textsf{rebuts}, \beta r) \bigr\rangle$\;
}
sort $P$ by $\mathrm{sc}$ descending, ties by ascending identifier\;
\Return the first $\min(K,|P|)$ items of $P$ with their $\mathrm{why}$ records\;
\end{algorithm}

Note that Algorithm~\ref{alg:rank} returns the decomposition alongside the
slate rather than offering it on request. Contestability that must be asked
for is contestability most participants never exercise; the interface of
Section~\ref{sec:browser} shows the terms next to each served item.

\subsection{The weight policy is the whole game}
\label{sec:weights}

Equation~\eqref{eq:score} has three policy parameters and they are not of equal
consequence. Varying $\alpha$ and $\beta$ across their plausible range moves
completeness by amounts we could not distinguish from noise in the attack-free
regime (Section~\ref{sec:res_null}); changing $\wrel$ as seen next moves it by
$0.08$--$0.12$ under coordinated attack (Section~\ref{sec:res_flood}). 
We compare two policies: the \emph{flat} policy
$\wrel^{\mathrm{flat}}(j,k) = 1$, and the \emph{author-normalised} policy
\begin{equation}
  \wrel^{\mathrm{norm}}(j,k) \;=\; \frac{|A(j,k)|}{\max\bigl(1, V^{\sigma(k)}\bigr)},
  \label{eq:norm}
\end{equation}
where $A(j,k)$ is the set of distinct participants who have asserted the link
$(j,k)$ and $V^{\sigma}$ is the number of participants who have cast a ballot
on side $\sigma$ so far. Stating the normalisation in terms of \emph{distinct
authors} is what keeps it compatible with C\ref{c:author}: it counts how many
separate people stand behind a claimed relation and is indifferent to which
people they are. Its effect on a coordinated coalition is direct. A coalition
of size $m$ can multiply the \emph{number} of links it asserts freely, but
under \eqref{eq:norm} each link is worth only the distinct authorship behind
it, so the coalition's total link credit scales with $m$ rather than with its
output; Section~\ref{sec:res_flood} measures the resulting difference at
$0.08$--$0.12$ of completeness. The obvious escalation is to co-sign --- have
all $m$ members assert the same links so that $|A(j,k)| = m$ and the numerator
recovers. Section~\ref{sec:res_cosign} reports what happens, and the result is
pleasantly counter-intuitive: completeness \emph{rises} monotonically with the
degree of co-signing, because co-signing concentrates the coalition's
identities on a small set of links instead of spreading them across many, so
the same authorship budget buys fewer distinct promoted items and the
coalition's own material crowds itself out. Table~\ref{tab:policy} gives the
two policies side by side.

\begin{table}[tb]
\centering\footnotesize
\caption{The two relation-weight policies. The last column previews
Section~\ref{sec:res_flood}: mean combined completeness under a {\bf flooding
coalition holding a quarter of the electorate}, at the reference configuration.}
\label{tab:policy}
\begin{tabular}{@{}>{\raggedright\arraybackslash}p{2.7cm}>{\raggedright\arraybackslash}p{4.6cm}>{\raggedright\arraybackslash}p{5.0cm}c@{}}
\toprule
\textbf{policy} & \textbf{$\wrel(j,k)$} & \textbf{reads} & \textbf{under flood} \\
\midrule
flat & $1$ & nothing beyond the link's existence & $0.34$ \\[3pt]
author-normalised & $|A(j,k)| \,/\, \max(1, V^{\sigma(k)})$ &
  how many distinct people asserted it, against how many voted on that side & $0.44$ \\
\bottomrule
\end{tabular}
\end{table}

\subsection{Configurability without incoherence}
\label{sec:config}

C\ref{c:config} asks that the policy belong to the participant. The immediate
worry is that this dissolves the shared object: if everyone sees a different
slate, what is the common record the deliberation is about? The resolution is a
matter of what varies and what does not. What varies is the policy vector
$\theta_i = (\alpha_i, \beta_i, \wrel^{(i)}, K_i)$, plus optionally a reason
emphasis $\mu_i$ in the sense of weighted SJP. A participant who wants the
strongest objections to positions they already hold sets $\beta$ high; one who
wants depth on established claims sets $\alpha$ high; one who distrusts
coordinated authorship adopts \eqref{eq:norm}; one who wants a longer slate
raises $K$. These are editorial stances, they are legitimately personal, and
none is the kind of thing an operator should be choosing on anyone's behalf\footnote{For the special case of so-called {\em grassroot organizations}, users can also influence voter eligibility coefficients~\cite{qin2013census}.}.
What does not vary is everything else: the corpus, the endorsement tallies, the
link set, the labels, and the rule itself. Two participants with different
$\theta$ compute different slates from \emph{identical} public evidence, and
each can compute the other's exactly. This is what distinguishes configuration
from personalisation: a personalised feed differs because the system holds a
private model of the user, whereas a configured slate differs because the user
set a published parameter, and the difference is fully explained by that
parameter. Under configuration a disagreement about slates reduces to a
disagreement about policy, arguable in public; under personalisation it reduces
to nothing at all.

\begin{proposition}[Charter closure under configuration]
\label{prop:charter}
If the rule satisfies C\ref{c:det}--C\ref{c:contest} for every fixed
$\theta$, then the family $\{\mathrm{sc}_{\theta}\}_{\theta \in \Theta}$
satisfies them for participant-chosen $\theta$, provided $\theta_i$ is
recorded alongside the slate. Determinism, locality, author blindness and
semantic abstinence hold pointwise in $\theta$ and are inherited;
reproducibility and contestability require the auditor to know which $\theta$
was in force, and persisting $(\theta_i, t_i)$ with each slate --- as
Section~\ref{sec:persist} does --- supplies it.
\end{proposition}

Proposition~\ref{prop:charter} is why Section~\ref{sec:persist} stores the
policy vector with every served slate, and why Section~\ref{sec:res_pareto}
can present its main finding as a \emph{frontier} rather than an optimum: once
the choice between ranking arms is a position on a trade-off between coverage
and endorsement mass, there is no system-level answer to which position is
right, and C\ref{c:config} is what lets the question be handed to the person
it affects.

\subsection{Rules of order as algorithms of interaction}
\label{sec:robert}

The idea that procedure constitutes rather than merely regulates deliberation
is older than any of this. Robert's \emph{Rules of Order}~\cite{robert1915rules}
are a specification for a distributed system with adversarial participants:
recognition of the floor is admission control; the motion stack is a
well-founded ordering guaranteeing termination; the requirement to speak to the
motion is a relevance predicate; limits on repeat speech are rate limiting; and
the prohibition on impugning motives is a type restriction on admissible
utterances. None of it was derived from a theory of good outcomes, but from the
observation that without such constraints an assembly is captured by whoever is
loudest and most persistent. Our construction translates the same instinct into
a setting where the floor is a screen: the published rule is the standing
order,
author blindness is the convention that the chair
recognises members rather than reputations, and the persistence of
Section~\ref{sec:persist} is the minutes. \DDP{}'s own rule that a motion
admits one active argumentation per voter at a time~\cite{silaghi2013ddp2p} is an
anti-filibuster device with an exact parliamentary ancestor, and formal work on
procedural argumentation makes the correspondence
explicit~\cite{prakken2001formalizing,silaghi2014encompassing}. The relevant
inheritance is not any particular clause but the stance: an assembly's rules
must be knowable in advance by everyone bound by them, which is a constraint on
the form of the rules, not a preference about their content.

\section{\ABAS{}: An Agentic Instrument for Auditable Deliberation}
\label{sec:abas}

Everything in Sections~\ref{sec:model} and~\ref{sec:charter} is a
specification. To find out what it does one needs an electorate, and human
electorates are not available for parameter sweeps. \ABAS{} ---
\emph{Agent-Based Argument Simulator} --- instantiates $N$ constituent agents
against a real implementation of the poll object, runs the round protocol to
completion, and persists enough state that every served slate can be
reconstructed afterwards. This section describes the agent, the round, the
model of submissions that fail to justify, the arms compared, what is stored,
and the interface through which a participant would inspect any of it. The
design commitment throughout is that the simulator is an \emph{instrument},
not a product demonstration: it is built so that its own outputs can be
audited, and Section~\ref{sec:persist} is where that commitment is cashed.

\subsection{The constituent agent}
\label{sec:agent}

Each agent holds a scalar opinion $\theta \in [-1,1]$ on the motion, drawn
once at initialisation, and a short position text generated from that opinion
by a topic-specific generator. The opinion determines the ballot direction
stochastically, so agents near the centre are genuinely uncertain; the
position text is what the agent uses to judge which existing material speaks
for it, through TF--IDF cosine
similarity~\cite{manning2008introduction,sparckjones1972idf}.

Two properties of the agent matter for how the results should be read. First,
agents do not update their opinion in response to what they are shown: the
simulator measures \emph{exposure}, not persuasion, and a model of persuasion
would import assumptions we have no basis for. Second, agent behaviour is a
model of constituent behaviour, not evidence about it, in the sense of
Section~\ref{sec:rw_agentic}; every claim we make is a claim about mechanism
under a stated model. Section~\ref{sec:limits} states the specific form this
caveat takes for each result, and in one case --- the degenerate-authoring
result --- it is load-bearing enough that we state it twice.

\subsection{The round protocol}
\label{sec:round}

Agents are processed in identifier order, each one completing the sequence of
Algorithm~\ref{alg:round} before the next begins. Because the corpus therefore
grows underneath the sequence, the live vocabulary of
Definition~\ref{def:live} genuinely differs between the first and last agent,
which is what makes the temporal inequity of Section~\ref{sec:exposure}
measurable rather than assumed.

\begin{algorithm}[tb]
\caption{One constituent's turn}
\label{alg:round}
\SetAlgoLined
\Given{agent $a$ with opinion $\theta_a$ and position text $x_a$;
       poll state at time $t$; policy $\theta$; budget $K$;
       action probabilities $\preuse, \pabst, \pown$ with
       $\preuse+\pabst+\pown = 1$; link probability $\plnk$}
$S^{+} \leftarrow$ \Rank{$\Jpos \cap \mathcal{J}_t$, $\theta$, $K$};\quad
$S^{-} \leftarrow$ \Rank{$\Jneg \cap \mathcal{J}_t$, $\theta$, $K$}\;
persist $\langle a, t, \theta, S^{+}, S^{-}\rangle$ \tcp*{before any ballot is cast}
$v \leftarrow$ ballot direction drawn from $\theta_a$\;
$u \leftarrow \mathcal{U}[0,1)$\;
\uIf{$u < \preuse$}{
  $j \leftarrow$ \Adopt{$S^{v}$, $x_a$} \tcp*{similarity-weighted draw from the served slate}
  record ballot for $v$ citing $j$; increment $E(j)$\;
}
\uElseIf{$u < \preuse + \pabst$}{
  record ballot for $v$ with no citation\;
}
\Else{
  \uIf{$a$ authors degenerately \emph{(Section~\ref{sec:degen})}}{
    $j \leftarrow$ \Degen{$x_a$, $v$}\;
  }\Else{
    $j \leftarrow$ \Compose{$x_a$, $v$} \tcp*{fresh item, labels from the side vocabulary}
  }
  insert $j$; record ballot for $v$ citing $j$\;
  \If{$\mathcal{U}[0,1) < \plnk$}{
    \LinkOut{$j$} \tcp*{one rebuttal at an opposing item, one reinforcement at a same-side item}
  }
}
\end{algorithm}

Adoption is restricted to the served slate --- an agent can only
cite what it was shown --- which couples the recommender to the endorsement
tallies and makes the system a feedback loop rather than a passive display.
And the authoring branch draws labels from the side's vocabulary independently
of the slate, so corpus growth is driven by the authoring draw and not by what
was served; this is why the ablations of Section~\ref{sec:instruments} are
clean, all ranking arms seeing a \emph{bit-identical} corpus and differing only
in which $K$ items of it were displayed. Link targets are chosen by the agent,
not the system: a rebuttal aims at the opposing item least similar to the
author's position, a reinforcement at the same-side item most similar to it.

\subsection{Submissions that fail to justify}
\label{sec:degen}

In the first set of experiments we report here, the authoring model had every agent who wrote anything write a
well-formed argument carrying two to five genuine reasons for the side it had
voted. Section~\ref{sec:res_null} reports the price of that charity: under it,
almost any twenty items cover almost everything, and there is nothing for a
ranking rule to do. Real deliberative corpora are not like that: a large share
of contributions state a position without giving a logically valid reason for it, and a
further share give what the author takes to be a reason for their position but
which is in fact a reason for the other side. We model both, with two
parameters.

\begin{description}[leftmargin=1.4em,itemsep=2pt]
\item[$G$ --- degenerate fraction.] The share of own-authored items that fail
  to justify. Swept over $\{0, 0.1, 0.3, 0.5, 0.7\}$.
\item[$E$ --- erroneous-support share.] Of those, the share that are
  \emph{erroneous supports} --- content belonging to the opposing side, filed
  under this one --- rather than merely \emph{reason-free}. Swept over
  $\{0.1, 0.2, 0.4\}$.
\end{description}

A reason-free item carries a position and no labels: it occupies a slate slot
and contributes nothing to any union. An erroneous support carries labels
belonging to the other side: it occupies a slot and
inflates the denominator with vocabulary that does not belong to the side it
was filed under.

Two implementation choices make the resulting comparisons trustworthy. First,
\emph{who} authors degenerately is drawn off a dedicated random stream, so
raising $G$ changes the content of affected items while leaving the corpus
skeleton --- which agents author, how many items exist, which links form, how
many ballots are cast --- bit-identical across the entire grid; every
comparison in Section~\ref{sec:res_degen}, across arms and across grid points
alike, is therefore exactly paired. Second, completeness is measured against
the \emph{genuine} vocabulary, a reason counting towards a side's denominator
only if it is a canonical reason for that side. This is the conservative
choice: counting opposite-side labels would let erroneous supports inflate the
very quantity they are supposed to damage and would reward a slate for
surfacing them. We record the denominator-inflating variant alongside so that
the size of the measurement artefact is visible rather than assumed away, and
Section~\ref{sec:res_degen} reports it.

\subsection{Ranking arms and the evaluation ceiling}
\label{sec:arms}

The arms compared throughout Sections~\ref{sec:res_honest}--\ref{sec:adversarial}
are identical in every respect except the ranking step of
Algorithm~\ref{alg:round}.

\begin{description}[leftmargin=1.4em,itemsep=2pt]
\item[\textsf{full}] The published rule, \eqref{eq:score} with both link
  summands active.
\item[\textsf{endorse-only}] $\alpha = \beta = 0$: a raw endorsement count.
\item[\textsf{enhance-only}] $\beta = 0$: reinforcement summand retained.
\item[\textsf{attack-only}] $\alpha = 0$: rebuttal summand retained.
\item[\textsf{random}] The score is ignored and the slate is drawn uniformly
  from the live pool, using a generator seeded independently of the main
  stream so that stances, actions and link choices are consumed identically
  across arms and the comparison stays seed-paired.
\end{description}

Alongside these we compute, on the evaluation path only and subject to
Principle~\ref{prin:oracle}, the greedy cover of Section~\ref{sec:oracle}.
Algorithm~\ref{alg:oracle} states it explicitly so that what it reads --- and
therefore why it is inadmissible --- is on the page.

\begin{algorithm}[tb]
\caption{\textnormal{\textsc{GreedyCover}} --- evaluation ceiling only; not a deployable mechanism}
\label{alg:oracle}
\SetAlgoLined
\Given{live pool $P$, labelling $\lab$ \emph{(read directly --- violates
       C\ref{c:abstain})}, live vocabulary $\Lam^{\sigma}_t$, budget $K$}
\Yields{a $K$-subset of $P$ and its completeness}
$S \leftarrow \emptyset$;\quad $C \leftarrow \emptyset$\;
\For{$1$ \KwTo $\min(K,|P|)$}{
  $j^{\star} \leftarrow \arg\max_{j \in P \setminus S} \bigl|\lab(j) \setminus C\bigr|$,
     ties by ascending identifier\;
  \lIf{$|\lab(j^{\star}) \setminus C| = 0$}{\textbf{break}}
  $S \leftarrow S \cup \{j^{\star}\}$;\quad $C \leftarrow C \cup \lab(j^{\star})$\;
}
\Return $S$ and $|C| / |\Lam^{\sigma}_t|$\;
\end{algorithm}

\subsection{Persistence: the record is the point}
\label{sec:persist}

Every run writes a single relational database holding: the agents and their
initialisation parameters; every justification with its author, side,
timestamp and labels; every link with its author, type and endpoints; every
ballot with its direction, citation and timestamp; the simulation parameters
and seed; and --- the entry that matters --- \emph{every served slate}, stored
with the identifier of the constituent it was served to, the instant, the
policy vector in force, and the item identifiers in served order.

This is what makes C\ref{c:repro} and C\ref{c:contest} testable rather than
asserted. Any historical slate can be reconstructed exactly, the corpus state
at that instant with it, and the score decomposition of every served item
recomputed and checked against the ranking actually served. It is also what
makes the ordering and endorsement-mass instruments of
Sections~\ref{sec:res_order} and~\ref{sec:res_pareto} computable at all, since
both read served position.
A pipeline logging
only the aggregate metric would have required the entire experiment to be
repeated.

\subsection{Inspecting the record}
\label{sec:browser}

A record nobody can read discharges nothing. \ABAS{} ships a browser over the
persisted database that presents, for any constituent: the two slates they
were served, in order; for each served item, the decomposition
$(\textsf{own}, E(j))$, $(\textsf{reinforces}, \alpha a)$ and
$(\textsf{rebuts}, \beta r)$ produced by Algorithm~\ref{alg:rank}; the
links that contributed each term, with their authors; and the ballot that
followed. It also presents the counterfactual a participant most often wants
--- the slate the same evidence would have produced under a different policy
vector --- which is C\ref{c:config} made concrete rather than promised. The
interface is deliberately unremarkable: tables and links, no visualisation of
anything the arithmetic does not literally contain. The design literature on
decentralised civic tools is consistent about non-expert users needing the
model of the system to be simple enough to hold in
mind~\cite{alqahtani2017hci,alqahtani2016cognition,kattamuri2005supporting},
and a three-term sum is about the limit of what one can expect a participant
to check while deciding how to vote.

\section{Experimental Setup}
\label{sec:setup}

This section fixes the protocol so that every number in
Sections~\ref{sec:res_honest}--\ref{sec:adversarial} can be located. Two
propositions are used throughout,
one reference configuration anchors all sweeps; the adversary
families are stated with the exact capability each is granted; and the
statistical conventions are declared in advance, including where we do not
correct for multiplicity and what we therefore do not claim. A third proposition is used for the refresh-interval scenario of Section~\ref{sec:res_cache}.

\subsection{Propositions}
\label{sec:topics}

\textbf{UBI} asks \emph{``Should one address the AI
revolution by introducing basic income?''}, with thirty canonical reasons per side spanning fiscal,
labour-market, administrative, distributive and political-economy
considerations. \textbf{BRA} asks \emph{``Should one address the AI revolution by introducing
basic resource assurance (minimal healthy housing, food, and emergency
healthcare)?''}.
A third proposition, \textbf{OPT-OUT}, asks \emph{``Should primary school pupils be able to study without computers and an internet connection?''}, with thirty canonical reasons per side spanning pedagogical, developmental, equity-of-access and administrative considerations. 
It is used only in Section~\ref{sec:res_cache}, where the question is whether a protocol parameter held fixed  everywhere else changes the conclusions.

\subsection{Reference configuration}
\label{sec:refconfig}

Unless a sweep states otherwise: $N = 1000$ constituents; action probabilities
$\pown = 0.10$, $\preuse = 0.50$, $\pabst = 0.40$; link probability
$\plnk = 0.60$; slate size $K = 20$ per side; author-normalised weight policy
\eqref{eq:norm}; policy coefficients $\alpha = \beta$ at the implementation
default; score caches re-read after every constituent ($M = 1$), the interval
swept in Section~\ref{sec:res_cache}.
Table~\ref{tab:baseline}
characterises what this produces and establishes that the regime is not
degenerate in either direction: the vocabulary saturates, the corpus is
roughly $50$ items per side against a $30$-label vocabulary, and the ballot
split is near even.

\begin{table}[tb]
\centering\footnotesize
\caption{The reference configuration characterised: BRA proposition,
$\pown = 0.10$, $K = 20$, $\plnk = 0.60$, $N = 1000$, author-normalised
$\wrel$; mean $\pm$ standard deviation over ten seeds. Completeness is that of
the slate the endorsement rule actually served. ``Within-run spread'' is the
standard deviation of per-constituent completeness \emph{inside} a single run,
averaged over seeds: it measures temporal inequity, not experimental noise,
and at $0.213$ it is more than five times the across-seed variation.}
\label{tab:baseline}
\begin{tabular}{@{}lrr@{}}
\toprule
Quantity & Mean & Std \\
\midrule
Combined completeness, mean over constituents   & 0.806 & 0.038 \\
Combined completeness, median over constituents & 0.886 & 0.042 \\
Within-run spread of completeness               & 0.213 & 0.032 \\
\midrule
Live endorsing vocabulary (of $30$)             & 30.0  & 0.0 \\
Live opposing vocabulary (of $30$)              & 29.9  & 0.3 \\
Authored endorsing justifications               & 48.3  & 6.6 \\
Authored opposing justifications                & 49.7  & 6.8 \\
\midrule
Rebuttal links                                  & 346   & 19 \\
Reinforcement links                             & 342   & 15 \\
Total links                                     & 688   & 33 \\
\midrule
Endorsing ballots                               & 500.9 & 15.1 \\
Opposing ballots                                & 499.1 & 15.1 \\
\bottomrule
\end{tabular}
\end{table}

\subsection{Structural levers of the non-degenerate authoring reference configuration}
\label{sec:sweeps}

In the first set of experiments with non-degenerate authoring, four structural levers are swept one at a time about the reference point,
ten seeds per cell, both propositions: authoring rate
$\pown \in \{0.02, 0.05, 0.10, 0.20, 0.30\}$ with $\preuse : \pabst$ held at
$5\!:\!4$; slate size $K \in \{5,10,15,20,30\}$; link rate
$\plnk \in \{0.20,0.40,0.60,0.80,1.00\}$; electorate
$N \in \{200,500,1000,2000,5000\}$. The ranking ablation runs five arms at
$K \in \{5,10,20,30\}$ with $100$ seed-paired seeds per cell
($4000$ runs) and repeats three arms at $K=20$ under attack ($1800$ runs). The
degenerate-authoring sweep runs three arms over
$G \in \{0,0.1,0.3,0.5,0.7\} \times E \in \{0.1,0.2,0.4\}$ with $50$ seeds per
cell per proposition ($3900$ runs). The adversarial families are described
next. In total the results below rest on roughly $17{,}000$ seeded runs.

\subsection{Adversary model}
\label{sec:adversary}

In this set of experiments, authentication is assumed to hold: the census is well-formed, nobody votes
twice, and no endorsement is forged. Sybil resistance is therefore \emph{out
of scope} as an attack and \emph{in scope} as an infrastructure requirement,
which is where Section~\ref{sec:p2p_census} takes it up. What a coalition
controls is what its members write, where they point their links, and when
they act. Coalition sizes are $0$, $5$, $10$, $15$, $20$ and $25$ percent of
$N$, and members are spread evenly through the processing order, so they enjoy
no first-mover endorsement cascade and must rely on the graph for visibility.  The attacks studied are:

\begin{description}[leftmargin=1.4em,itemsep=3pt]
\item[Hub-riding.] Members author items in the ordinary way and attach
  reinforcement links to the most endorsed same-side item, attempting to
  launder its endorsement mass into their own score.
\item[Label flooding.] Members author on every turn, drawing both labels from
  a single shared pair, link on every turn, and aim those links at the two
  most endorsed items. The clones carry real corpus labels, so any damage is
  redundancy, not label absence.
\item[Heterogeneous control.] Identical to flooding in authoring rate, link
  rate and link targeting, but each member draws its label pair independently:
  the rate- and corpus-matched control isolating homogeneity.
  A second variant restores
  most-related link targeting, isolating hub aiming as a residual.
\item[Co-signed flooding.] The coalition is partitioned into groups of size $C$;
  each group publishes one poison item per side and every member endorses it
  and re-asserts the same links, driving the distinct-author numerator of
  \eqref{eq:norm} from one to the group's per-side size.
  $C \in \{1,2,5,10,25,50\}$ at coalitions of a tenth and a quarter, both
  propositions, both weight policies ($4800$ runs). $C=1$ reproduces the
  flooding sweep exactly, to the last stored digit, and serves as an in-family
  control.
\end{description}

Nothing in the implementation obstructs any of these strategies. In
particular the relation store admits repeated $(\textit{from},\textit{to})$
assertions from distinct authors, which is exactly what makes the co-signing
numerator movable --- we did not defend against the attack by refusing to
represent it.

\subsection{Statistical conventions}
\label{sec:stats}

Comparisons between arms are seed-paired and tested with a two-sided paired
$t$-test on per-seed means; comparisons between configurations that do not
share seeds use Welch's unequal-variance test. Reported intervals are
across-seed standard deviations unless stated. Where many contrasts are
computed on the same stored runs --- notably the twenty-four ordering
contrasts of Section~\ref{sec:res_order} --- we report the raw $p$-values
without multiplicity correction and say so at the point of use, and any
contrast whose significance would not survive a Bonferroni adjustment at that
family size is described as \emph{suggestive} rather than established. We
have not adjusted for the number of hypotheses across the manuscript as a
whole, and readers should treat effects at $p \approx 10^{-2}$ accordingly;
the effects we build arguments on sit at $p < 10^{-10}$.

\section{Results I: An Attack-Free Electorate}
\label{sec:res_honest}

We begin with an electorate in which every
participant who writes anything writes a competent argument for the side they
support (non-degenerate authoring). This regime answers two questions well and one question misleadingly.
It establishes which structural levers govern how much of the live reason
vocabulary a served slate covers --- and the answer, that everything which
matters works by enriching the corpus before the voter arrives rather than by
selecting more cleverly from it, is stable across both propositions. It
establishes what the charter of Section~\ref{sec:charter} costs, by measuring
the served slates against a label-reading ceiling that upper-bounds every
mechanism including the ones the charter excludes. 

\subsection{What the reference configuration produces}
\label{sec:res_ref}

Table~\ref{tab:baseline} gives the characterisation. Mean combined
completeness is $0.806 \pm 0.038$ on BRA and $0.804 \pm 0.032$ on UBI: a
constituent at the reference configuration meets, on average, four fifths of
the reasons that existed on each side at the moment they voted.

The number that deserves more attention is the within-run spread, $0.213$.
This is the standard deviation of completeness \emph{across constituents
inside a single run}, and it is more than five times the standard deviation
\emph{between} runs. The dominant source of variation in what a voter sees is
therefore not the seed, the configuration, or the recommender --- it is
\emph{when in the sequence the voter arrived}. Early constituents vote against
a vocabulary that is still assembling itself; late ones vote against a mature
one. This is the temporal inequity of Section~\ref{sec:exposure}, measured, and
it is an inequality in the informational conditions of the ballot that no
choice of selection rule addresses. Figure~\ref{fig:maturation} shows how the
saturation point moves with authoring rate, and the area to the left of each
curve is the population that voted early enough for it to matter.

\begin{figure}[tb]
\centering
\includegraphics[width=0.52\linewidth]{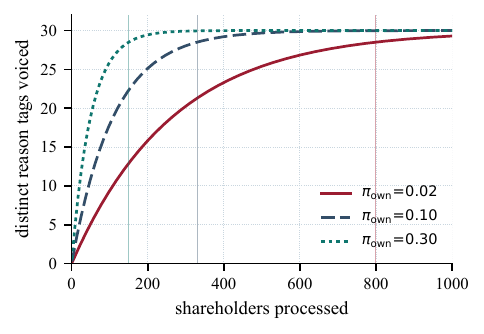}
\caption{Vocabulary saturation under three authoring rates, anchored at the
measured saturation points ($\approx 800$ constituents at $\pown = 0.02$,
$\approx 150$ at $\pown = 0.30$, in runs of $N = 1000$); the vertical rules
mark the measured anchors and the shape between them is illustrative. The area
to the left of each curve is the population that voted against a reason space
still under construction --- the quantity a deployment reduces by raising the
authoring rate, not by ranking better.}
\label{fig:maturation}
\end{figure}

\subsection{Which levers move coverage}
\label{sec:res_levers}

Table~\ref{tab:allresults} gives every setting of every lever on both
propositions. The ordering is unambiguous and identical across them.

Slate size spans widest --- $0.44$ at $K=5$ to $0.83$ at $K=30$ --- and is
bounded above by the corpus, since a slate cannot cover reasons nobody has
written. The fifth through fifteenth slots contribute the overwhelming
majority; slots beyond the twentieth add two to three points per block, and
only because the rule is label-blind, so lower-ranked but reason-novel items
keep entering. A label-reading selector would exhaust the vocabulary earlier
and then add nothing.
Authoring rate follows: raising $\pown$ from $0.02$ to
$0.30$ lifts completeness from $0.70$ to $0.88$, by the mechanism visible in
Figure~\ref{fig:maturation} --- a higher rate saturates the vocabulary sooner,
so a smaller fraction of the electorate votes against an immature reason
space. It is the cheapest lever per point gained, and one a deployment pulls
through interface design and prompting rather than through algorithms.
Electorate size behaves identically, $0.69$ at $N=200$ rising to $0.90$ at
$N=5000$, with the difference that it is not a design parameter; it is worth
stating because it means the construction gets \emph{better} at scale, the
opposite of the usual finding for deliberative procedures.

Among these levers, link rate does essentially nothing: completeness moves by $0.004$ across a
fivefold change in linking activity, on both propositions. This is the first
appearance of a fact Section~\ref{sec:res_null} will sharpen considerably. It
does not mean links are useless --- links are what the rebuttal and
reinforcement summands read, and Section~\ref{sec:res_degen} shows they carry
real weight once the corpus contains material worth discriminating against. It
means that in a corpus where every item is a competent argument, the ordering
induced by the links reshuffles items that were going to be served anyway.

\begin{table}[tb]
\centering\footnotesize
\caption{Mean combined completeness $\bar{c}$ at every setting of every lever,
both propositions, author-normalised weight policy, ten seeds per cell.
Settings run left to right in the order given in the row label; the reference
setting is in bold. Across-seed $\sigma$ runs $0.013$--$0.082$ throughout and
is tabulated per cell in the online appendix. In the authoring-rate row
$\preuse = 0.50$ is fixed and $\pabst$ absorbs the remainder; all other levers
are held at reference.}
\label{tab:allresults}
\begin{tabular}{@{}ll rrrrr@{}}
\toprule
Lever (settings) & prop. & \multicolumn{5}{c}{$\bar{c}$ by setting} \\
\midrule
\multirow{2}{*}{$\pown \in \{0.02, 0.05, \mathbf{0.10}, 0.20, 0.30\}$}
  & BRA & 0.702 & 0.713 & \textbf{0.806} & 0.841 & 0.878 \\
  & UBI & 0.684 & 0.751 & \textbf{0.804} & 0.853 & 0.873 \\
\midrule
\multirow{2}{*}{$K \in \{5, 10, 15, \mathbf{20}, 30\}$}
  & BRA & 0.448 & 0.666 & 0.763 & \textbf{0.806} & 0.833 \\
  & UBI & 0.448 & 0.666 & 0.768 & \textbf{0.804} & 0.832 \\
\midrule
\multirow{2}{*}{$\plnk \in \{0.20, 0.40, \mathbf{0.60}, 0.80, 1.00\}$}
  & BRA & 0.806 & 0.804 & \textbf{0.806} & 0.804 & 0.804 \\
  & UBI & 0.799 & 0.801 & \textbf{0.804} & 0.805 & 0.802 \\
\midrule
\multirow{2}{*}{$N \in \{200, 500, \mathbf{1000}, 2000, 5000\}$}
  & BRA & 0.691 & 0.744 & \textbf{0.806} & 0.850 & 0.904 \\
  & UBI & 0.680 & 0.732 & \textbf{0.804} & 0.865 & 0.894 \\
\bottomrule
\end{tabular}
\end{table}

\subsection{The price of semantic abstinence, measured}
\label{sec:res_ceiling}

Section~\ref{sec:oracle} introduced the greedy cover as an evaluation ceiling
and Principle~\ref{prin:oracle} confined it to the evaluation path. Here we
spend it. For every slate served in the reference cell we also compute, over
the \emph{identical} pool and at the \emph{identical} cache refresh, the slate
a label-reading greedy selector would have returned, scoring both with the same
measure. This answers the question the charter leaves open: not whether some
excluded procedure could score higher, but whether the admissible class
contains a procedure good enough to use.

It does. Against the end-of-round vocabulary (Table~\ref{tab:ceiling}) the
endorsement rule serves $0.806 \pm 0.038$ where the ceiling reaches
$0.837 \pm 0.034$: a gap of $0.031 \pm 0.014$, or $3.7\%$ of the ceiling,
positive in all ten seeds and never exceeding $0.059$. Reading the labels is
worth about three points of completeness.

The second block is the more consequential reading. Scored against the
vocabulary actually live when each slate was served --- the denominator of
Definition~\ref{def:completeness} --- the ceiling attains $1.000$ in every seed
and the served slates $0.968 \pm 0.015$. The ceiling saturates because the
reference cell ends with a thirty-label vocabulary per side against a
twenty-item slate, so a spanning sub-collection almost always exists; the
$(1-e^{-1})$ worst case of Proposition~\ref{prop:greedy} is nowhere near
binding at this scale. The $0.194$ shortfall from $\bar{c} = 1$ therefore
decomposes into two very unequal parts: a \emph{selection} component of
$0.031$, which is what abstaining from the labels costs, and a \emph{temporal}
component of $0.163$ --- vocabulary that did not yet exist when the slate was
assembled, which no procedure of any kind, rule-based or learned or
label-omniscient, could have shown that voter. Four fifths of the measured
incompleteness is a property of sequential deliberation rather than of the
recommender.

This governs how the rest of the manuscript should be read. Every lever of
Section~\ref{sec:res_levers} that moves completeness substantially does so by
attacking the temporal component; the selection component is small and is the
only one any better \emph{selection} procedure could recover, so a learned
ranker offered the same pool could at best close $0.031$, and only by reading
what C\ref{c:abstain} forbids.

\begin{table}[tb]
\centering\footnotesize
\caption{Served completeness against the label-reading greedy ceiling,
reference configuration (BRA), summarised over ten seeds; the per-seed rows are
in the online appendix. The ceiling is evaluated on the same pool and at the
same cache refresh as the served slate, so the gap isolates selection quality
rather than index staleness. The left block scores against the end-of-round
vocabulary, the right block against the vocabulary live at the moment of
serving (Definition~\ref{def:completeness}). The ceiling saturates the live
vocabulary in every seed, so the two gaps almost
coincide and the residual shortfall
in the left block is temporal rather than algorithmic.}
\label{tab:ceiling}
\begin{tabular}{@{}l rrr rrr@{}}
\toprule
& \multicolumn{3}{c}{End-of-round vocabulary}
& \multicolumn{3}{c}{Live vocabulary} \\
\cmidrule(lr){2-4}\cmidrule(lr){5-7}
& Served & Ceiling & Gap & Served & Ceiling & Gap \\
\midrule
Mean over seeds & 0.806 & 0.837 & 0.031 & 0.968 & 1.000 & 0.032 \\
Std over seeds  & 0.038 & 0.034 & 0.014 & 0.015 & 0.000 & 0.015 \\
Minimum         & 0.750 & 0.793 & 0.011 & 0.940 & 1.000 & 0.011 \\
Maximum         & 0.862 & 0.888 & 0.059 & 0.989 & 1.000 & 0.060 \\
\bottomrule
\end{tabular}
\end{table}

\section{Results II: What the Coverage Measure Could Not See}
\label{sec:instruments}

This section reports the ablation
 --- does the rule beat not ranking at all?
--- and the answer, on set coverage with non-degenerate authoring, is no. It reports that null in full, with
the seeds on which the ranked and random slates were bit-identical, because a
manuscript arguing for legibility owes its readers the result that embarrassed
it. It then shows that the null is an artefact of two things: an instrument
that discards order by construction, and an authoring model in which every
submission is worth reading. Correcting either one dissolves it. Correcting
both reveals that the two link summands, which the ablation found exactly
inert, are worth as much as the entire ranking advantage once the corpus
contains material a rule should keep off the slate. The section closes on the
third instrument, endorsement mass, which shows that the random baseline was
never competitive on any axis except the one we happened to be measuring.

\subsection{The null result, for non-degenerate authoring}
\label{sec:res_null}

Five arms, identical in every respect but the ranking step
(Section~\ref{sec:arms}), $100$ seed-paired seeds per cell, four slate sizes,
both propositions --- $4000$ runs. One property of the model makes the
comparison unusually clean: corpus growth is driven by the authoring draw and
not by the slate, so all arms see a \emph{bit-identical} corpus ($102.4$ items
and $694$ links at zero attackers, $326.4$ and $1035$ at a quarter-electorate
coalition), and the served slate is the only thing that differs.

\paragraph{The link terms are all but inert.} Table~\ref{tab:ablation} gives the
non-degenerate authoring attack-free sweep. No two of the four ranked arms separate by more than $0.004$, and
of the twenty-four seed-paired contrasts between ranked arms only one survives a
Bonferroni correction for that many comparisons: at $K=5$ the enhance-only arm
leads the pure endorsement count by $0.0039$ ($p=0.001$). At that same slate
size $58$ of the $100$ BRA seeds are bit-identical between the full rule and the
pure endorsement count; at $K=20$ about a fifth of seeds still are.

The mechanism is arithmetic rather than mysterious. Under author normalisation
a relation weight is the number of distinct authors of an edge divided by the
constituents voting that side, which at the reference configuration ranges
from $0.002$ to $0.023$; multiplied by the endorsement count of a hub it
yields a bonus that never exceeded $0.96$ in the state we instrumented. The
gap in raw endorsements at the $K$-boundary was $3$ in that same state. A
bonus smaller than the boundary gap reorders items \emph{within} the slate
without changing which items are in it --- and completeness, by
Remark~\ref{rem:orderblind}, reads only the union of what is served. The link
terms are not doing nothing. They are doing something the instrument is
constitutionally unable to see.

\begin{table}[tb]
\centering\footnotesize
\caption{Ranking-term ablation in an non-degenerate authoring attack-free electorate: mean combined
completeness pooled over the two propositions, $100$ seed-paired seeds per
cell. The live pool holds about $51$ items per side throughout, so $K$ is also
a proxy for the fraction of the pool served. Across-seed $\sigma$ runs
$0.026$--$0.049$ for the four ranked arms and $0.015$--$0.025$ for the random
arm.}
\label{tab:ablation}
\begin{tabular}{@{}r rrrr r@{}}
\toprule
$K$ & full & endorse-only & enhance-only & attack-only & random \\
\midrule
$5$  & 0.4500 & 0.4470 & 0.4509 & 0.4498 & 0.4475 \\
$10$ & 0.6706 & 0.6708 & 0.6712 & 0.6723 & 0.6652 \\
$20$ & 0.8205 & 0.8199 & 0.8202 & 0.8213 & 0.8181 \\
$30$ & 0.8511 & 0.8505 & 0.8510 & 0.8510 & 0.8506 \\
\bottomrule
\end{tabular}
\end{table}

\paragraph{The random baseline is not beaten, and under attack it wins.} 
At non-degenerate authoring and zero attackers the rule leads by $0.8205$ against $0.8181$, a
margin of $0.0024$ that is suggestive at $p=0.034$ and does not survive
correction for the number of contrasts; at a tenth of the electorate the rule
\emph{trails}, $0.6088$ against $0.6472$ ($-0.0385$, $p \approx
5\times10^{-27}$); at a quarter it trails further, $0.3731$ against $0.4437$
($-0.0706$, $p \approx 3\times10^{-39}$), all seed-paired over $200$ runs per
comparison. The direction is the mechanism rather than an anomaly: the ranked
arms read the endorsement tallies the coalition inflates, and a uniform draw
does not, so the clones the flood manufactures reach the slate through the
ranking step and not around it. The three-way comparison at the largest
coalition is the one to sit with: the author-normalised policy reaches $0.3695$
on BRA and $0.3768$ on UBI, the random slate $0.4419$ and $0.4456$, the flat
policy $0.2279$ and $0.2195$. Read at face value the flat policy is some
twenty-two points \emph{worse} than not ranking at all, and author normalisation
buys back about two thirds of that deficit but not the whole of it, leaving the
rule seven points short of the uniform draw --- its virtue under attack is the
limiting of amplification rather than the delivery of coverage. The random arm
also carries less than half the across-seed variance of the ranked arms under
attack ($\sigma \approx 0.035$ against $0.078$), an effect we do not currently
model.

modelling assumption responsible. A learned selector exhibiting the same null
would have offered a metric that failed to move and an unbounded space of explanations.

\subsection{Order: what a ranking rule is actually for}
\label{sec:res_order}

Completeness is invariant to slate order (Remark~\ref{rem:orderblind}), and a
ranking rule is precisely a claim about order. 
$m \in \{1,2,3,5,10,20\}$, the area under the prefix curve, rank-discounted
coverage, and $e_{90}$, at $K=20$, pooled over both propositions and
seed-paired across arms.

Table~\ref{tab:ordering} and Figure~\ref{fig:prefix} give the result, and it
is unambiguous in exactly the way Table~\ref{tab:ablation} was not. At every
prefix short of the full slate the endorsement rule leads the uniform draw,
and the lead grows with adversarial pressure. In an attack-free electorate the
gains are small but uniformly significant: $p_1$ $+0.0094$
($p \approx 10^{-10}$), $p_5$ $+0.0198$ ($7\times10^{-15}$), $\AUC$ $+0.0135$
($3\times10^{-19}$), $\RDC$ $+0.0347$ ($10^{-17}$), $e_{90}$ $-0.40$ positions
($5\times10^{-10}$). At a coalition of a tenth of the electorate they become
$p_5$ $+0.1016$, $\AUC$ $+0.0558$ and $e_{90}$ $-3.14$ positions, all at
$p \approx 2\times10^{-34}$; at a quarter, $p_1$ $+0.0335$, $p_5$ $+0.1387$,
$\AUC$ $+0.0796$, $\RDC$ $+0.1788$ and $e_{90}$ $-5.69$ positions, all at
$p \le 5\times10^{-34}$.

The reading is direct. A constituent under a quarter-electorate coalition who
reads the ranked slate reaches nine tenths of that slate's coverage after
$8.9$ items; one reading the uniform draw needs $14.6$. Both slates end at the
same $p_{20}$ --- which is why Table~\ref{tab:ablation} saw nothing --- but
they are not the same object for anyone with finite attention. The coalition
succeeds at filling the corpus and fails at placing its material where the
reader is. That the margin \emph{grows} under attack inverts the natural
expectation that an adversary manipulating the graph would degrade the rule's
ordering advantage. What happens instead is that as the pool fills with
redundant clones, the difference between ordering by endorsement flow and not
ordering at all becomes the difference between a reader meeting distinct
reasons early and a reader wading through duplicates --- a distinction that
barely exists in a clean pool and dominates in a polluted one.

The link terms remain nearly invisible on this instrument too. Comparing the
full rule against \textsf{endorse-only}, no contrast reaches significance at
zero or a tenth attackers; at a quarter, $\RDC$ gains $0.013$
($p = 0.011$) and $\AUC$ gains $0.004$ ($p = 0.027$). With twenty-four
contrasts computed on the same stored runs and no multiplicity correction
(Section~\ref{sec:stats}), we report this as \emph{suggestive} and nothing
more. The next subsection is where the link terms stop being suggestive.

\begin{figure}[tb]
\centering
\includegraphics[width=\linewidth]{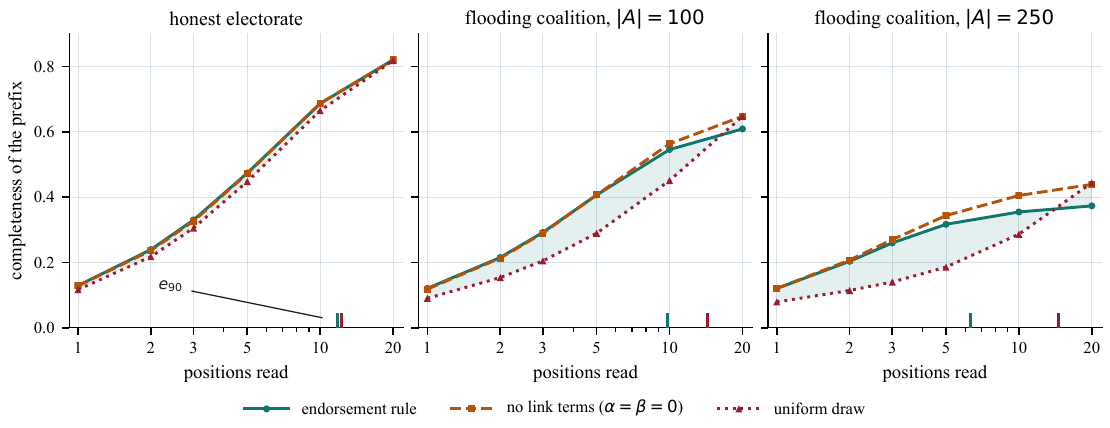}
\caption{Prefix coverage against slate depth at $K=20$, pooled over both
propositions, $2800$ stored runs. Panels: attack-free electorate, coalition at a
tenth, coalition at a quarter. The shaded region is what the endorsement rule
delivers above the uniform draw at each depth; ticks on the axis mark $e_{90}$,
the depth at which each arm reaches nine tenths of its own final coverage. The
three arms meet at $m = 20$ in the attack-free panel and separate everywhere before
it, by a margin that widens with attack; under attack the endorsement rule ends
$m = 20$ below the draw --- the only quantity Table~\ref{tab:ablation} could
see --- and leads it at every earlier depth.}
\label{fig:prefix}
\end{figure}

\begin{table}[tb]
\centering\footnotesize
\caption{Ordering measures at $K=20$, pooled over both propositions,
seed-paired, $2800$ stored runs. $p_m$ is prefix coverage at depth $m$;
$\AUC$ is the mean over depths; $\RDC$ is rank-discounted coverage
\eqref{eq:rdc}; $e_{90}$ is the depth reaching nine tenths of that slate's own
final coverage (lower is better). $|A|$ is the coalition size out of
$N = 1000$. Note that $p_{20}$ --- the only column completeness can see ---
agrees across arms to within $0.005$ in the attack-free electorate, and under
attack moves against the ranked rule while every earlier column moves for it.}
\label{tab:ordering}
\begin{tabular}{@{}rl rrrr rr r r@{}}
\toprule
$|A|$ & arm & $p_1$ & $p_3$ & $p_5$ & $p_{10}$ & $\AUC$ & $\RDC$ & $e_{90}$ & $p_{20}$ \\
\midrule
\multirow{3}{*}{$0$}
 & full         & 0.1290 & 0.3301 & 0.4738 & 0.6865 & 0.4466 & 1.3298 & 11.74 & 0.8205 \\
 & endorse-only & 0.1281 & 0.3267 & 0.4720 & 0.6872 & 0.4451 & 1.3258 & 11.68 & 0.8199 \\
 & random       & 0.1165 & 0.3054 & 0.4476 & 0.6650 & 0.4284 & 1.2833 & 12.28 & 0.8181 \\
\midrule
\multirow{3}{*}{$100$}
 & full         & 0.1205 & 0.2919 & 0.4066 & 0.5457 & 0.3648 & 1.0705 & \phantom{0}9.79 & 0.6088 \\
 & endorse-only & 0.1182 & 0.2898 & 0.4068 & 0.5635 & 0.3732 & 1.1021 & 10.49 & 0.6474 \\
 & random       & 0.0901 & 0.2048 & 0.2894 & 0.4509 & 0.3060 & 0.9544 & 14.38 & 0.6472 \\
\midrule
\multirow{3}{*}{$250$}
 & full         & 0.1200 & 0.2597 & 0.3167 & 0.3545 & 0.2712 & 0.7884 & \phantom{0}6.29 & 0.3731 \\
 & endorse-only & 0.1201 & 0.2699 & 0.3435 & 0.4050 & 0.2974 & 0.8670 & \phantom{0}7.97 & 0.4389 \\
 & random       & 0.0790 & 0.1400 & 0.1857 & 0.2864 & 0.2082 & 0.6676 & 14.61 & 0.4437 \\
\bottomrule
\end{tabular}
\end{table}

\subsection{Degenerate authoring: giving the rule something to discriminate against}
\label{sec:res_degen}

The second correction addresses Remark~\ref{rem:charityblind}. Under the model
of Section~\ref{sec:res_honest} every submission is a competent argument, so
every item is worth roughly the same to a coverage measure and no rule can
beat a draw by much. Section~\ref{sec:degen} introduced two parameters --- the
degenerate fraction $G$ and the erroneous-support share $E$ --- and this
subsection sweeps them: three arms, $G \in \{0,0.1,0.3,0.5,0.7\}$,
$E \in \{0.1,0.2,0.4\}$, $50$ seeds per cell per proposition, $3900$ runs, zero
failures, twenty minutes of wall clock. Because degenerate authorship is drawn
off a dedicated stream, the corpus skeleton is bit-identical across the whole
grid and every comparison is exactly paired.

Table~\ref{tab:degenerate} gives the result: the margin the coverage
instrument could not find in Table~\ref{tab:ablation} reappears as soon as
there is anything to discriminate against, and grows monotonically with how
much there is. At $G = 0$ the arms are all but indistinguishable, matching the attack-free sweep
of Table~\ref{tab:ablation} ($+0.0033$, $p = 0.018$) --- the control
establishing that the sweep measures what it claims. At $G = 0.1$ the rule
already leads the uniform draw by $0.0059$ ($p = 7\times10^{-4}$); by
$G = 0.7$ it leads by $0.0471$ ($p = 8\times10^{-14}$), more than an order of
magnitude larger. The served degenerate share tracks it: at $G=0.7$ the rule
serves $66.6\%$ degenerate material against the $70.5\%$ corpus base rate a
uniform draw returns, significant at $p < 10^{-6}$ in every cell.

Raising $E$ --- moving degenerate items from reason-free towards
erroneous-support --- consistently \emph{shrinks} the margin, by $0.005$ to
$0.010$ across the grid. This is the expected direction and a useful sanity
check: an erroneous support carries real labels and real endorsement
potential, so a rule reading only endorsements and links finds it harder to
distinguish from a genuine item than a reason-free stub. The rule's advantage
is largest exactly where the material is most obviously worthless. We also
record the measurement artefact rather than assuming it away: scoring against
the denominator-inflating variant at $G=0.3$, $E=0.4$ gives $0.6227$ where the
genuine-vocabulary measure gives $0.7424$, because erroneous supports push the
apparent per-side vocabulary from $29.6$ labels to $44.8$. Reporting the raw
variant would have made every arm look worse and credited slates for surfacing
material filed on the wrong side; the genuine-vocabulary measure of
Section~\ref{sec:degen} is the conservative choice and is what
Table~\ref{tab:degenerate} reports throughout.

Ordering and coverage now agree. At $G = 0.7$, $E = 0.1$ the rule reaches
$p_5 = 0.4158$ against the draw's $0.1862$, $\AUC$ $0.3629$ against $0.2104$,
and $e_{90}$ of $7.39$ against $13.25$. The two instruments that disagreed in
Sections~\ref{sec:res_null} and~\ref{sec:res_order} disagree because of the
authoring model, and once it is corrected they say the same thing.

\begin{table}[tb]
\centering\footnotesize
\caption{Degenerate authoring: mean combined completeness against the
\emph{genuine} reason vocabulary, pooled over both propositions, $100$ runs
per cell, seed-paired across arms. $G$ is the fraction of own-authored items
that fail to justify; $E$ is the share of those that are erroneous supports
rather than reason-free. $p$-values are two-sided Wilcoxon signed-rank tests on
the full-minus-random contrast. The $G=0$ row reproduces the attack-free sweep
of Table~\ref{tab:ablation} and serves as the control.}
\label{tab:degenerate}
\begin{tabular}{@{}cc rrr rc r@{}}
\toprule
$G$ & $E$ & full & endorse-only & random & full $-$ random & $p$ & full $-$ endorse \\
\midrule
$0.0$ & ---   & 0.8232 & 0.8233 & 0.8199 & $+0.0033$ & $0.018$             & $-0.0000$ \\
\midrule
$0.1$ & $0.1$ & 0.7996 & 0.7982 & 0.7940 & $+0.0056$ & $2.3\times10^{-3}$  & $+0.0014$ \\
$0.1$ & $0.4$ & 0.7993 & 0.7982 & 0.7940 & $+0.0053$ & $5.5\times10^{-3}$  & $+0.0011$ \\
$0.3$ & $0.1$ & 0.7479 & 0.7323 & 0.7268 & $+0.0210$ & $3.6\times10^{-12}$ & $+0.0156$ \\
$0.3$ & $0.4$ & 0.7424 & 0.7323 & 0.7268 & $+0.0156$ & $9.2\times10^{-9}$  & $+0.0102$ \\
$0.5$ & $0.1$ & 0.6639 & 0.6242 & 0.6248 & $+0.0392$ & $4.9\times10^{-15}$ & $+0.0397$ \\
$0.5$ & $0.4$ & 0.6538 & 0.6242 & 0.6248 & $+0.0290$ & $4.5\times10^{-12}$ & $+0.0295$ \\
$0.7$ & $0.1$ & 0.5764 & 0.5147 & 0.5104 & $+0.0661$ & $2.9\times10^{-16}$ & $+0.0617$ \\
$0.7$ & $0.4$ & 0.5596 & 0.5147 & 0.5104 & $+0.0492$ & $1.9\times10^{-14}$ & $+0.0449$ \\
\bottomrule
\end{tabular}
\end{table}

\subsection{Why the link terms square the discrimination}
\label{sec:res_mechanism}

Table~\ref{tab:degenerate}'s last column is the surprise. The link summands,
worth $-0.0000$ at $G=0$, are worth $+0.0617$ at $G=0.7$ --- which is
$93\%$ of the full margin over the uniform draw. Terms that
Table~\ref{tab:ablation} found all but inert become, in a realistic corpus,
the substance of the rule. This subsection explains why, by measuring the
graph directly.

Recall the direction of \eqref{eq:score}: an item is credited for the
endorsements of what it \emph{points at}. So the question is whether a
degenerate author's outgoing links land on lower-endorsed targets than a
genuine author's do. Instrumenting $100$ runs at $G = 0.7$, $E = 0.4$ gives
the answer in one table of five numbers per item class:

\begin{center}\small
\begin{tabular}{@{}lrrrrr@{}}
\toprule
item class & count & own endorsements & reinforcement out-edges & per item & link credit \\
\midrule
genuine           & 30.2 & 11.04 & 196.1 & 6.5 & 274.25 \\
erroneous support & 28.6 & \phantom{0}6.95 & 117.5 & 4.1 & 136.28 \\
reason-free       & 43.6 & \phantom{0}1.43 & \phantom{0}34.2 & 0.8 & \phantom{0}38.12 \\
\bottomrule
\end{tabular}
\end{center}

On endorsements alone, genuine items lead
reason-free ones by a factor of $7.7$. On link credit they lead by a factor of
$7.2$ --- and the two factors \emph{multiply}, because an item's link credit is
the sum of its targets' endorsements and a degenerate item both asserts fewer
links ($0.8$ per item against $6.5$) and aims the few it asserts at material
nobody adopted. The link summands do not introduce new discrimination. They
\emph{square} the discrimination already present in the endorsement count.

This also explains the inertness at $G=0$ without special pleading. When every
item is genuine, every item has both a healthy endorsement count and healthy
out-edges into other well-endorsed items; squaring a ratio of one leaves one.
The link terms were never a mechanism for separating good arguments from good
arguments. They are a mechanism for separating arguments from non-arguments,
and we had not given them any non-arguments.

One limitation must be stated here rather than deferred, because it governs
the reading of the entire subsection. The chain rests on a single modelled
fact --- degenerate items are adopted less ($1.43$ endorsements against
$11.04$) --- and that fact follows from the simulator's adoption step, which
selects from the served slate by TF--IDF cosine similarity to the agent's
position text (Section~\ref{sec:agent}). It is a property of the agents, not
of the rule. The finding is therefore that \emph{the rule
inherits and amplifies whatever discrimination the electorate itself
exercises}. 
The estimation of the rate at which human electorate discriminates is an empirical question this manuscript does not attempt to answer (Section~\ref{sec:limits}).

\subsection{Coverage against endorsement mass}
\label{sec:res_pareto}

The third instrument closes the account. Sections~\ref{sec:res_null}
and~\ref{sec:res_order} left an uncomfortable residue: on coverage the uniform
draw was never beaten, and under flooding attacks it beat the rule.
Instrument~III \eqref{eq:mass} asks the question
coverage cannot --- \emph{did the slate show the voter what the electorate had
actually taken up?} --- and the answer is that the draw was never competitive.

Table~\ref{tab:pareto} gives five regimes: the
attack-free non-degenerate corpus, two degenerate-authoring levels, and two coalition sizes. The
uniform draw captures $0.49$ of the achievable endorsement mass in the attack-free
regime against the rule's $0.76$, and $0.17$ against $0.55$ at a
quarter-electorate coalition --- a factor of $3.3$. Its coverage --- level with
the rule in the attack-free regime, ahead of it under attack --- is purchased
entirely by serving material nobody had adopted. A voter shown the random
slate met at least as many reasons, attached to items their fellow
constituents had, in the main, passed over.

The greedy ceiling behaves in a way worth dwelling on. It attains coverage
$1.0000$ in every regime, as Section~\ref{sec:res_ceiling} led one to expect.
On endorsement mass it \emph{dominates} the rule in the three attack-free
regimes ($0.79$ against $0.76$ at $G=0$) and is \emph{dominated} under attack
($0.38$ against $0.55$ at a quarter-electorate coalition). Under flooding, the
label-reading selector buys its perfect coverage by abandoning what
constituents endorsed --- it reaches for the rare labels, which under a
homogeneity attack are the ones the coalition has not saturated and which
almost nobody has adopted. The inadmissible mechanism is not merely
inadmissible; on the axis that measures whether a slate reflects the
electorate, it is worse under exactly the conditions a civic deployment must
survive.

Finally, at $G = 0.5$ the link terms buy $+2.8$ points of coverage for $-2.4$
points of endorsement mass. That is not an error and not a defect. It is a
frontier: the two link summands trade breadth of reasons against fidelity to
what people endorsed, and there is no system-level fact about which is
correct. It is exactly the kind of question C\ref{c:config} exists to hand to
the person it affects, and Section~\ref{sec:discussion} takes up what it means
that our clearest quantitative finding turns out to be a menu rather than an
answer.

\begin{table}[tb]
\centering\footnotesize
\caption{Coverage against endorsement mass \eqref{eq:mass} across five
regimes, pooled over both propositions. $G$ is the degenerate fraction at
$E=0.1$; $|A|$ is coalition size out of $N=1000$. The greedy ceiling attains
coverage $1.0000$ in every regime, so only its mass is shown. The uniform draw
is level with the rule on coverage in the attack-free regime and ahead of it under
attack, and loses on mass everywhere, by a
factor rising to $3.3$; the ceiling dominates on mass in the three attack-free
regimes and is dominated in both attacked ones.}
\label{tab:pareto}
\begin{tabular}{@{}l ccc ccc c@{}}
\toprule
& \multicolumn{3}{c}{\textbf{coverage}} & \multicolumn{3}{c}{\textbf{endorsement mass}} & \textbf{ceiling} \\
\cmidrule(lr){2-4}\cmidrule(lr){5-7}
regime & full & end.-only & random & full & end.-only & random & mass \\
\midrule
$G = 0$, no coalition & 0.8232 & 0.8233 & 0.8199 & 0.7560 & 0.7601 & 0.4906 & 0.7905 \\
$G = 0.3$       & 0.7479 & 0.7323 & 0.7268 & 0.7458 & 0.7610 & 0.4915 & 0.7645 \\
$G = 0.7$       & 0.5764 & 0.5147 & 0.5104 & 0.7215 & 0.7590 & 0.4906 & 0.7920 \\
$|A| = 100$     & 0.6088 & 0.6474 & 0.6472 & 0.6895 & 0.6943 & 0.2916 & 0.5875 \\
$|A| = 250$     & 0.3731 & 0.4389 & 0.4437 & 0.5460 & 0.5579 & 0.1673 & 0.3809 \\
\bottomrule
\end{tabular}
\end{table}

\section{Results III: A Coalition in the Electorate}
\label{sec:adversarial}

We now admit a coordinated coalition under the model of
Section~\ref{sec:adversary}: authentication holds, so nobody votes twice and
no endorsement is forged, and the coalition's only instruments are what it
writes, where it points its links, and when it acts. The two natural
strategies produce opposite outcomes, and the contrast is the most useful
result in this manuscript for anyone actually deploying such a system --- it
says that the intuitive attack is harmless, that the unintuitive one is
severe, and that the defence against the severe one is a single line of the
specification. We then decompose the damage to establish what causes it, and
give the coalition its best available escalation to establish that the defence
does not merely displace the problem.

\subsection{Hub-riding is inert}
\label{sec:res_hub}

When coalition members author ordinary items and attach reinforcement
links to the leading same-side item, in order to launder its endorsement mass
into their own score, completeness does not move. At every coalition size, on
both propositions, the change from the zero-attacker cell is statistically
indistinguishable from zero ($n=200$ seeds per cell, all $p > 0.50$ under
Welch's test), and the two weight policies differ by less than $0.002$.

The reason is structural. Completeness is a property of a union of sets, and
exchanging one ordinary item for another
changes which elements contribute to the union without changing the union
much. Hub-riding achieves its proximate goal --- the attacker's argument gets
seen --- while leaving the informational quality of the slate intact. Whether
this constitutes an attack at all is a fair question: a system in which
arguing that your point extends a popular point makes your point more visible
is arguably working as designed, and this is one of the few places where the
right response to an adversarial finding is to accept the behaviour rather
than defend against it.

The result also confirms the design choice of Section~\ref{sec:rule}
empirically. \eqref{eq:score} takes exactly one hop and computes no fixed
point, so there is no eigenvector to farm~\cite{page1999pagerank}; a coalition
that constructs a hub and points its whole corpus at it gains a bounded,
one-hop bonus and nothing compounding.

\subsection{Label flooding, and the weight policy as a security control}
\label{sec:res_flood}

The picture inverts entirely when the clones are label-identical.
Table~\ref{tab:attackers} gives the sweep, $200$
seeds per cell, spread timing so the coalition enjoys no first-mover cascade.
Under the author-normalised policy completeness falls from $0.816$ at zero
attackers to $0.377$ at a quarter of the electorate on BRA, and $0.815$ to
$0.372$ on UBI. Under the flat policy the same coalition drives it to $0.228$
and $0.219$. Every non-zero cell differs from its baseline at $p < .001$, and
every normalised-versus-flat comparison at a non-zero coalition size differs
at $p < .001$ as well.

The defence margin --- the completeness that author normalisation retains and
a flat policy loses --- is $0.117$ at a coalition of five percent, widens to
$0.181$ at fifteen and eases to $0.151$ at twenty-five. Section~\ref{sec:weights} claimed the
relation-weight function is a security control rather than a tuning knob;
this is the measurement behind the claim, and it should be read against
Section~\ref{sec:res_levers}, where the four structural levers of the attack-free
regime were the only things that moved completeness at all.

The arithmetic behind the defence is Example~\ref{ex:bus} scaled up. Under the
flat policy a lone clone's link into the hub transfers the hub's entire
endorsement mass, so every clone inherits a score comparable to the most
successful non-coalition item in the pool and the top of the ranking fills with
duplicates. Under author normalisation the same link transfers a fraction
$|A(j,k)| / V^{\sigma}$; a solo clone inherits about a thousandth of the hub's
mass. To buy what the flat policy gives away, the coalition must make many
members assert the \emph{same} link --- which means the cost of the attack
scales with its size, and, more importantly, that the attack becomes
\emph{visible} in the public link record as an anomalous concentration of
identical assertions. A defence that converts a covert manipulation into an
overt one is doing exactly what a civic mechanism should do, and note that it
does so without any term that reads \emph{which} participants asserted the
link: \eqref{eq:norm} counts distinct authors and remains compliant with
C\ref{c:author}.

\begin{table}[tb]
\centering\footnotesize
\caption{Flooding attack under spread timing: mean combined completeness
against coalition size, both weight policies, both propositions, $200$ seeds
per cell. The final column gives the defence margin --- how much completeness
author normalisation retains that a flat policy loses --- averaged over the
two propositions. Hub-riding, not shown, produces no significant change at any
coalition size.}
\label{tab:attackers}
\begin{tabular}{@{}r rr rr r@{}}
\toprule
& \multicolumn{2}{c}{\textbf{BRA}} & \multicolumn{2}{c}{\textbf{UBI}} & \\
\cmidrule(lr){2-3}\cmidrule(lr){4-5}
Coalition & normalised & flat & normalised & flat & margin \\
\midrule
$0\%$  & 0.816 & 0.814 & 0.815 & 0.814 & --- \\
$5\%$  & 0.713 & 0.597 & 0.708 & 0.590 & $0.117$ \\
$10\%$ & 0.608 & 0.440 & 0.602 & 0.427 & $0.172$ \\
$15\%$ & 0.519 & 0.340 & 0.514 & 0.331 & $0.181$ \\
$20\%$ & 0.443 & 0.282 & 0.430 & 0.265 & $0.163$ \\
$25\%$ & 0.377 & 0.228 & 0.372 & 0.219 & $0.151$ \\
\bottomrule
\end{tabular}
\end{table}

\subsection{The mean over sides hides how one-sided the damage is}
\label{sec:res_sidebalance}

One caveat applies to both aggregators in Equations~\ref{eq:mincompleteness} and~\ref{eq:meancompleteness}. About $2\%$ of constituents --- $19.6$
of $987$ at the reference configuration --- receive a slate on one side only
and so contribute a single value, whereas the factor $1/2N$ in
\eqref{eq:meancompleteness} presumes two. The figure we report as $\bar{c}$
throughout is the mean over slates actually served, which coincides with
\eqref{eq:meancompleteness} when every voter is served on both sides and
otherwise sits $0.003$--$0.007$ below it, the one-sided voters being those
whose remaining side is the sparse one; $\bar{c}^{\min}$ is necessarily
restricted to the voters served on both. Both discrepancies are an order of
magnitude smaller than any effect discussed here, but the asymmetry is worth
naming: an unserved side is a coverage failure that neither equation counts as
one.

The substantive conclusions are unchanged. The minimum tracks the mean at an
offset of $0.04$--$0.08$ across the whole grid, the damage it registers is
$5$--$13\%$ larger than the mean's at every coalition size, and the advantage
of author normalisation over a flat policy is not merely preserved but wider
under it, $0.133$--$0.188$ against $0.118$--$0.182$, at the same significance.
The choice of aggregator does not decide any claim we make.

What it does change is what a reader can see.
Table~\ref{tab:sidebalance} reports, for the same runs as
Table~\ref{tab:attackers}, the mean over voters of the worse-covered side, the
mean within-voter gap between the two sides, and the share of constituents
whose worse side falls below one half. The gap roughly doubles under attack,
from $0.085$ with no coalition to $0.15$--$0.16$ at the largest ones: flooding
does not lower coverage evenly, it unbalances it. The last column is the
consequence. At a $15\%$ coalition under author normalisation the mean reads
$0.52$, which invites the reading that a typical voter still meets half of the
reasons; in those same runs $58\%$ of constituents have a side below $0.5$,
and under a flat policy $94\%$ do. Both statements are true
of the same electorate, and only the first is visible in
\eqref{eq:meancompleteness}.

\begin{table}[!htbp]
\centering
\footnotesize
\caption{Side imbalance under label flooding, on the runs of
Table~\ref{tab:attackers} (200 seeds per cell, pooled over the two
propositions). ``min'' is $\bar{c}^{\min}$ of \eqref{eq:mincompleteness},
``gap'' the mean over voters of the difference between their two sides, and
``$<\!0.5$'' the share of voters whose worse side falls below one half.}
\label{tab:sidebalance}
\begin{tabular}{@{}rcccccc@{}}
\toprule
& \multicolumn{3}{c}{author-normalised} & \multicolumn{3}{c}{flat} \\
\cmidrule(lr){2-4}\cmidrule(lr){5-7}
Coalition & min & gap & $<\!0.5$ & min & gap & $<\!0.5$ \\
\midrule
$0\%$  & 0.779 & 0.085 & 0.118 & 0.777 & 0.086 & 0.118 \\
$5\%$  & 0.661 & 0.107 & 0.130 & 0.528 & 0.138 & 0.348 \\
$10\%$ & 0.538 & 0.140 & 0.257 & 0.351 & 0.168 & 0.782 \\
$15\%$ & 0.447 & 0.144 & 0.581 & 0.259 & 0.155 & 0.941 \\
$20\%$ & 0.359 & 0.158 & 0.878 & 0.194 & 0.160 & 0.987 \\
$25\%$ & 0.300 & 0.150 & 0.978 & 0.149 & 0.149 & 0.997 \\
\bottomrule
\end{tabular}
\end{table}

We keep \eqref{eq:meancompleteness} as the reported instrument. It is linear, so a drop in it decomposes into the per-voter and per-side contributions that make an attack attributable to a subpopulation, and it is the expectation of the coverage seen by a reader who stops at a uniformly random point, which is the quantity the three measures are jointly framed around; the minimum has neither property. But the mean supports a narrower reading than it appears to. A completeness of $0.52$ is a statement about an average side, not about an average voter, and an adversary who concentrates on one side is rewarded by exactly that difference. Where a deployment sets a threshold below which it will not certify a poll, the threshold belongs on the imbalance, not on the mean.

\subsection{The refresh interval matters only when someone is attacking}
\label{sec:res_cache}

Score caches are re-read after every constituent (Section~\ref{sec:refconfig}),
so no slate is scored against a snapshot older than the vote before it. Widening
the interval to $M$ constituents makes it the latency of the endorsement
feedback loop, and that lag reinterprets both preceding results at once: a
coalition accumulating endorsement mass gets $M$ votes of head start before the
quantity it inflates is read again, and a rule scoring against a stale snapshot
is not ranking the corpus the constituent is about to see. What lags is the
tallies and the relation weights, not the corpus: a newly authored item is
appended and forces a re-rank at any $M$. Crossing $M = 200$ against the
reported $M = 1$ with the coalition makes the objection an estimable
interaction. The experiment was run on a third proposition,
\textbf{OPT-OUT} (\emph{should primary pupils be able to study without computers
and an internet connection?}, thirty canonical reasons per side spanning
pedagogical, developmental, equity-of-access and administrative
considerations), and is reported on its own rather than pooled with the two of
Section~\ref{sec:topics}: $M \in \{1, 200\}$ crossed with no coalition against
the co-signed flood of Section~\ref{sec:res_cosign} at a tenth of the electorate
in groups of $C = 5$, three arms, the two extremes of the degenerate axis, $50$
seeds shared across all four corners, $1200$ runs.

Table~\ref{tab:cache} supports two conclusions of different kinds. With no
coalition the interval changes nothing the ordering claims rest on: on the
non-degenerate pool \textsf{full}$-$\textsf{random} is $-0.0011$ at $M = 1$ and
$-0.0015$ at $M = 200$ (both $p > 0.5$), both consistent with the
attack-free contrast of Table~\ref{tab:ordering} ($+0.0024$, suggestive
only); under degenerate authoring the rule's
advantage is $+0.0222$ and $+0.0202$ ($4\times10^{-6}$, $1\times10^{-5}$) ---
the same effect whether the loop is fresh or two hundred votes stale. With the
coalition the interval matters, in the unhelpful direction: lagging the refresh
to $M = 200$ \emph{raises} the ranked arm by $0.013$ of completeness and leaves
the uniform draw untouched. The interaction --- the effect of the lag on
\textsf{full}$-$\textsf{random} under attack minus the same effect without ---
is $+0.0136$ on the non-degenerate pool (95\% CI $[+0.0062, +0.0209]$,
$p = 5\times10^{-4}$) and $+0.0069$ under degenerate authoring
(CI $[-0.0011, +0.0150]$, $p = 0.089$).

\begin{table}[!htbp]
\centering\scriptsize
\caption{Completeness at the four corners of the refresh-interval $\times$
coalition results table, OPT-OUT, $50$ seeds per cell, author-normalised weights,
co-signed flood at $C = 5$. $M$ is the number of constituents between cache
refreshes; $M = 1$ is the reference configuration. The
\textsf{full}$-$\textsf{random} rows are seed-paired differences with two-sided
paired $t$-tests. $E$ is inert at $G = 0$.}
\label{tab:cache}
\setlength{\tabcolsep}{4pt}
\begin{tabular}{@{}ll rr rr@{}}
\toprule
& & \multicolumn{2}{c}{\textbf{no coalition}}
& \multicolumn{2}{c}{\textbf{$10\%$ coalition}} \\
\cmidrule(lr){3-4}\cmidrule(lr){5-6}
Load & arm & $M = 1$ & $M = 200$ & $M = 1$ & $M = 200$ \\
\midrule
\multirow{4}{*}{$G = 0$}
 & full                    & 0.815 & 0.815 & 0.707 & 0.720 \\
 & endorse-only            & 0.815 & 0.815 & 0.717 & 0.729 \\
 & random                  & 0.816 & 0.816 & 0.742 & 0.742 \\
 & \emph{full$-$random}    & $-0.0011$ & $-0.0015$
                           & $-0.0348^{***}$ & $-0.0216^{***}$ \\
\midrule
\multirow{4}{*}{$G = 0.5$, $E = 0.4$}
 & full                    & 0.642 & 0.640 & 0.534 & 0.539 \\
 & endorse-only            & 0.625 & 0.625 & 0.527 & 0.538 \\
 & random                  & 0.620 & 0.620 & 0.552 & 0.552 \\
 & \emph{full$-$random}    & $+0.0222^{***}$ & $+0.0202^{***}$
                           & $-0.0174^{**}$ & $-0.0124^{*}$ \\
\bottomrule
\end{tabular}
\end{table}

The mechanism is the order-invariance of Remark~\ref{rem:orderblind}, visible
slate by slate; Table~\ref{tab:cachearm} isolates each arm. \textsf{random}
never consults the cached scores, and its slates are identical item for item and
position for position across intervals at every corner, so its completeness is
exactly equal on all $50$ seeds. \textsf{endorse-only} is exactly invariant too,
but only without a coalition: over ten seeds its $19{,}702$ slates have
identical membership at $M = 1$ and $M = 200$ while $93\%$ are served in a
different order, which completeness cannot see. Under the coalition that breaks
--- at seed $3$ only $801$ of $1988$ slates keep their membership --- and
completeness moves $+0.0120$ ($5\times10^{-10}$). What the flood supplies is not
staleness but \emph{speed}: it makes tallies move far enough within one window
for the stale snapshot to select a different twenty items. \textsf{full} also
reads relation weights, whose staleness changes which items are served on $40\%$
of slates at $G = 0$ and $49\%$ at $G = 0.5$ even with no attacker, but there
the coverage gained and lost cancels ($-0.0004$, $-0.0020$, both $p > 0.35$).

\begin{table}[!htbp]
\centering\scriptsize
\caption{Cost of lagging the refresh interval, per arm: mean seed-paired
$M = 200$ minus $M = 1$ completeness, OPT-OUT, $50$ seeds. \emph{exact} marks
cells in which the two intervals produce the same completeness on every seed to
the last stored digit, so no test applies.}
\label{tab:cachearm}
\setlength{\tabcolsep}{4pt}
\begin{tabular}{@{}ll rl rl@{}}
\toprule
& & \multicolumn{2}{c}{\textbf{no coalition}}
& \multicolumn{2}{c}{\textbf{$10\%$ coalition}} \\
\cmidrule(lr){3-4}\cmidrule(lr){5-6}
Load & arm & $\Delta$ & $p$ & $\Delta$ & $p$ \\
\midrule
\multirow{3}{*}{$G = 0$}
 & full         & $-0.0004$ & $0.79$ & $+0.0132$ & $3\times10^{-4}$ \\
 & endorse-only & $\phantom{+}0.0000$ & exact & $+0.0120$ & $5\times10^{-10}$ \\
 & random       & $\phantom{+}0.0000$ & exact & $\phantom{+}0.0000$ & exact \\
\midrule
\multirow{3}{*}{$G = 0.5$, $E = 0.4$}
 & full         & $-0.0020$ & $0.36$ & $+0.0050$ & $0.17$ \\
 & endorse-only & $\phantom{+}0.0000$ & exact & $+0.0110$ & $4\times10^{-5}$ \\
 & random       & $\phantom{+}0.0000$ & exact & $\phantom{+}0.0000$ & exact \\
\bottomrule
\end{tabular}
\end{table}

The prefix measures agree at both intervals. Ranking's lead peaks at $p_5$ in
five of the eight corner-by-load cells and between $p_2$ and $p_{10}$ in the
rest: under degenerate authoring it is $+0.1363$ against $+0.1133$ with no
coalition and $+0.0786$ against $+0.0652$ with one. The lead is lost before the
full slate in one cell only --- the non-degenerate pool under attack at $M = 200$,
$+0.0011$ at $p_5$ and $-0.0006$ at $p_{10}$, the corner leaving a ranking rule
least to discriminate between. Scoring each constituent on the worse-served of
their two sides rather than averaging leaves the pattern intact:
\textsf{full}$-$\textsf{random} on that side is $-0.0365$ and $-0.0222$ under
attack on the non-degenerate pool, $+0.0208$ and $+0.0213$ under degenerate authoring
without one (all $p < 5\times10^{-4}$), and the gap between a constituent's two
slates is null in all eight cells ($p \ge 0.26$).

Two contrasts, at $p = 0.043$ and $p = 0.089$, are the kind
Section~\ref{sec:stats} asks be read as suggestive; the null without a
coalition, the interaction on the non-degenerate pool and the exact invariance of the
unranked arm do not depend on them. For the charter the reading is a small
negative one, worth stating because the opposite is the intuitive guess. $M$ is
a published policy parameter like \eqref{eq:norm}, and refreshing after every
vote sounds like the setting that keeps up best with an attack. It is not: under
a coalition the shorter interval carries the flood into the ranked slates
slightly faster and leaves the uniform draw exactly where it was, and with no
coalition it changes nothing. Reporting at $M = 1$ is thus the conservative
choice --- at $M = 200$ the rule would have looked $0.013$ better under attack
than it does --- and the defence remains the line that counts distinct authors,
not the rate at which its inputs are read.

\subsection{The collapse is homogeneity, not hyperactivity}
\label{sec:res_confound}

A flooding attacker differs from a non-coalition constituent in four ways at once,
and only one of them is the claim. It authors on every turn rather than with
probability $\pown$; it authors the \emph{same} narrow label set as every other
member; it links on every turn rather than with probability $\plnk$; and it
aims its links at the two most endorsed items rather than at the most related
one. The first and third inflate the corpus and the link graph, the fourth
concentrates endorsement mass, and any of them could depress completeness
alone, so comparing a coalition against an attacker-free baseline confounds
all four. We therefore ran the two rate-matched controls of
Section~\ref{sec:adversary}, at every coalition size, on both propositions,
under both weight policies, $100$ seeds per cell, seed-paired with the flooding
sweep ($4000$ runs). The match is tight --- at a tenth of the electorate on BRA
the flooding arm produces $190.6$ items and $837$ links against the
heterogeneous control's $191.6$ and $834$ --- so the arms differ in label
distribution and in nothing else the pipeline can observe.

Table~\ref{tab:confound} gives the decomposition. Hyperactivity is real but
small: a coalition holding a quarter of the electorate that authors and links
on every turn with heterogeneous labels costs $0.067$ of completeness under
normalisation, for an innocuous reason --- tripling the corpus enlarges the
denominator $|\Lam^{\sigma}|$ faster than a fixed twenty-item slate can cover
it. Homogeneity costs $0.380$ on top of that. Across coalition sizes and both
weight policies, between $78\%$ and $85\%$ of the collapse is attributable to
label homogeneity alone, and the share \emph{rises} with coalition size:
volume damage saturates while homogeneity damage does not. Hub aiming,
isolated by the third arm, is negligible under normalisation ($\le 0.006$ at
every coalition size, all $p \ge 0.14$) and worth at most $0.024$ under the flat
policy --- consistent with the mechanism, since author normalisation is
precisely the rule that discounts many identical assertions aimed at one
target. Two findings now agree from opposite directions:
Section~\ref{sec:res_levers} found link volume harmless, and this control
finds authoring volume largely harmless too. What the pipeline cannot absorb
is many voices saying one thing.

\begin{table}[tb]
\centering\footnotesize
\caption{Decomposing the flooding collapse, pooled over the two propositions,
$100$ seeds per cell. \emph{Heterogeneous} is the rate- and corpus-matched
control in which attackers author and link at the same inflated rates and aim
at the same hubs but draw labels independently. $\Delta_{\mathrm{vol}}$ is
completeness lost to sheer volume (baseline minus heterogeneous);
$\Delta_{\mathrm{hom}}$ is the additional loss attributable to label
homogeneity alone (heterogeneous minus flooding). Every $\Delta_{\mathrm{hom}}$
is significant at $p < 10^{-40}$ under Welch's test.}
\label{tab:confound}
\begin{tabular}{@{}r rrr rr r@{}}
\toprule
& \multicolumn{3}{c}{\textbf{mean completeness}}
& \multicolumn{2}{c}{\textbf{decomposition}} & \\
\cmidrule(lr){2-4}\cmidrule(lr){5-6}
Coalition & heterog. & flooding & (baseline)
& $\Delta_{\mathrm{vol}}$ & $\Delta_{\mathrm{hom}}$ & homog.\ share \\
\midrule
\multicolumn{7}{@{}l}{\emph{author-normalised weights}}\\
$5\%$  & 0.797 & 0.711 & 0.821 & $0.024$ & $0.086$ & $79\%$ \\
$10\%$ & 0.780 & 0.609 & 0.821 & $0.040$ & $0.172$ & $81\%$ \\
$15\%$ & 0.772 & 0.519 & 0.821 & $0.049$ & $0.253$ & $84\%$ \\
$20\%$ & 0.762 & 0.435 & 0.821 & $0.058$ & $0.328$ & $85\%$ \\
$25\%$ & 0.753 & 0.373 & 0.821 & $0.067$ & $0.380$ & $85\%$ \\
\midrule
\multicolumn{7}{@{}l}{\emph{flat weights}}\\
$5\%$  & 0.769 & 0.595 & 0.818 & $0.049$ & $0.174$ & $78\%$ \\
$10\%$ & 0.750 & 0.437 & 0.818 & $0.068$ & $0.313$ & $82\%$ \\
$15\%$ & 0.743 & 0.341 & 0.818 & $0.074$ & $0.402$ & $84\%$ \\
$20\%$ & 0.737 & 0.275 & 0.818 & $0.081$ & $0.462$ & $85\%$ \\
$25\%$ & 0.731 & 0.224 & 0.818 & $0.087$ & $0.507$ & $85\%$ \\
\bottomrule
\end{tabular}
\end{table}

\subsection{The coalition cannot coordinate its way out}
\label{sec:res_cosign}

The defence argument of Section~\ref{sec:res_flood} describes what a coalition
\emph{must} do to defeat normalisation; it does not show that doing it fails.
A co-signing coalition is strictly more powerful than one that does not,
since co-signing costs it nothing in membership, so we gave it the capability.
The coalition is partitioned into groups of $C$; each group publishes one
poison item per side, and every member endorses that item and re-asserts the
same two links from it, driving the distinct-author numerator of
\eqref{eq:norm} from one to the group's per-side size. $C=1$ is the attack of
Table~\ref{tab:attackers} and reproduces its per-seed completeness exactly, to
the last stored digit, across all eight cells --- an in-family control. We
swept $C \in \{1,2,5,10,25,50\}$ at coalitions of a tenth and a quarter, both
propositions, both weight policies, $4800$ runs.

Table~\ref{tab:cosign} reports the cleanest finding in the manuscript:
coordination is strictly counterproductive. Completeness rises monotonically
in $C$ at every coalition size, on both propositions, under both weight
policies. The uncoordinated $C=1$ attack is the coalition's optimum within the
family, and a fully coordinated coalition holding a quarter of the electorate
leaves completeness at $0.759$ against an attacker-free baseline of $0.815$
--- it has spent a quarter of the electorate to buy six points.

The reason is a rate mismatch the coalition cannot escape. Raising the
numerator requires spending members on one item; members are finite; so the
number of payload-carrying items falls as $1/C$ while the co-signature count
rises only linearly, and it rises against a denominator --- constituents
voting that side --- that co-signing does not touch. Measured on BRA at the
$25\%$ coalition, going from $C=1$ to $C=50$ multiplies mean co-signers per
edge by $3.4$ and the most concentrated edge by $1.8$, while cutting poison
items from $250.1$ to $10.0$, a factor of $25$. The non-coalition corpus is unmoved
throughout ($76.3$ items at $C=1$, $76.7$ at $C=50$). Completeness is damaged
by items occupying slate slots, and the coalition has traded away twenty-five
of those for every threefold gain in transfer.

\begin{table}[tb]
\centering\footnotesize
\caption{The coordinated adversary, pooled over the two propositions, $100$
seeds per cell. Group size $C$ is the number of coalition members sharing one
poison item and co-signing its links; $C=1$ is the uncoordinated attack of
Table~\ref{tab:attackers}. Payload items and co-signers per edge are measured
on BRA at the $25\%$ coalition. Completeness \emph{rises} monotonically with
$C$ in every cell: coordination helps the electorate, not the coalition.}
\label{tab:cosign}
\begin{tabular}{@{}r rr rr rr@{}}
\toprule
& \multicolumn{2}{c}{\textbf{coalition's position}}
& \multicolumn{2}{c}{\textbf{$10\%$ coalition}}
& \multicolumn{2}{c}{\textbf{$25\%$ coalition}} \\
\cmidrule(lr){2-3}\cmidrule(lr){4-5}\cmidrule(lr){6-7}
$C$ & payload items & co-signers/edge
& normalised & flat & normalised & flat \\
\midrule
$1$  & $250$ & $1.6$ & 0.609 & 0.437 & 0.373 & 0.224 \\
$2$  & $188$ & $2.0$ & 0.644 & 0.494 & 0.425 & 0.268 \\
$5$  & $97$  & $3.0$ & 0.710 & 0.629 & 0.530 & 0.404 \\
$10$ & $50$  & $4.1$ & 0.755 & 0.735 & 0.612 & 0.537 \\
$25$ & $20$  & $5.3$ & 0.785 & 0.781 & 0.722 & 0.711 \\
$50$ & $10$  & $5.5$ & 0.793 & 0.792 & 0.759 & 0.754 \\
\bottomrule
\end{tabular}
\end{table}

Note also what happens to the gap between the two weight policies as $C$
rises: it closes, from $0.149$ at $C=1$ to $0.005$ at $C=50$ at the $25\%$
coalition. This is the expected signature. Normalisation exists to discount
solo assertions; when the coalition co-signs everything, the two policies are
computing nearly the same quantity, and the coalition has arrived at a regime
where its own concentration is what limits it.

\subsection{What the adversarial results say about the charter}
\label{sec:res_charter}

Three consequences, each an argument for legibility rather than an argument
that merely happens to be compatible with it. First, the effective defence is
a \emph{published policy parameter}, not a detector: nothing in this section
trains a classifier, scores a participant's trustworthiness, or removes
anyone's material. \eqref{eq:norm} is one line, it counts distinct authors,
and it is worth ten points of completeness against a coalition holding a
quarter of the electorate. A deployment can publish it, a participant can
verify it was applied, and an adversary who reads it learns only that the
attack is expensive --- the property one wants a published defence to have.

Second, the defence works by making the attack \emph{visible}. A coalition
that pays the cost of co-signing writes its coordination into the public link
record, where an anomalous concentration of identical assertions from distinct
authors is directly observable by anyone holding a replica. A learned ranker
that resisted the same attack would do so through parameters nobody can
inspect, leaving participants unable to distinguish successful defence from
successful attack.

Third, the results relocate rather than remove the residual risk. Under authentication
the surviving exposure is not manipulation of the ranking --- hub-riding is
inert, co-signing is self-defeating --- but \emph{corpus capture}: a coalition
large enough simply crowds the reason space with redundancy, and
Section~\ref{sec:res_confound} shows four fifths of the damage comes from that
homogeneity rather than from anything the selector does. The corresponding
mitigations are not ranking mitigations. They are census integrity
(Section~\ref{sec:p2p_census}), which is what bounds coalition size at all,
and slate capacity, which Section~\ref{sec:res_levers} shows is the widest
lever available. A recommender is not where this class of attack should be
defeated, and a charter-compliant one at least makes that visible on the record.

\section{Peer-to-Peer Realisation on \DDP{}}
\label{sec:p2p}

The charter of Section~\ref{sec:charter} can be honoured on a central server,
but only as a promise. Determinism, evidence locality and reproducibility are
properties of a \emph{computation}, and if one party owns the machine that
performs it, participants are trusting that party's word about what ran ---
and configurability granted by an operator is revocable by the same operator.
The remedy is architectural: give each participant a replica of the items they
care about and let them evaluate the rule themselves.
Proposition~\ref{prop:partial} already established that this is possible,
because \eqref{eq:score} reads only an item and its out-neighbours. This
section shows that \DDP{}~\cite{silaghi2013ddp2p,silaghi2013petition} supplies
almost everything else: what the platform provides, how the model's objects
land on its item types, the identity problem that decentralisation makes
harder, evaluation over an incomplete replica, and what the empirical findings
of Sections~\ref{sec:instruments}--\ref{sec:adversarial} imply for a peer
deployment specifically.

\paragraph{What the platform provides.}
\label{sec:p2p_overview}
\DDP{} is a project developing an open-source platform for decentralised deliberative petition
drives, built around commitments that turn out to be the ones the charter
needs. \emph{Items are self-contained and globally identified}: every peer
record, organisation, constituent, motion, justification, signature, vote,
witnessing statement, news item and translation carries everything needed to
interpret it and is named by an identifier derived from a public key with a
creation date, or from a digest of its
content~\cite{silaghi2013ddp2p}. Identifiers compose hierarchically --- an
organisation's derives from its founding parameters, a motion's from the
organisation and its text, a justification's from the motion and its own text,
a vote's from the motion, the constituent and the justification cited --- so
two peers who have never communicated agree on the name of everything and
semantically distinct items cannot collide. \emph{Synchronisation is push--pull
gossip with a horizon}~\cite{silaghi2013petition,demers1987epidemic}: a request
carries the requester's interests and a time horizon and the answer carries
items newer than that horizon, with no global index and no authoritative
replica. \emph{Connectivity needs no owned infrastructure}: directory servers
help peers find each other and data servers hold items for offline peers, but
neither is trusted, both are replaceable, and their content is verifiable
against signatures; whether a peer relays for others is under that peer's own
control, by explicit design~\cite{alhamed2014supernode}, and mobile ad hoc and
vehicular operation have been studied as extreme cases of the same
idea~\cite{dhannoon2013vanet,dhannoon2013thesis}. Finally, \emph{organisations
are rules rather than accounts}: an organisation is a definition of a
constituency and a jurisdiction, constituencies may be defined recursively with
membership settled by a membership referendum --- a fixpoint grassroots
organisations resolve bottom-up rather than by administrative
decree~\cite{silaghi2013ddp2p} --- motions are Robert's motions
(Section~\ref{sec:robert}), and justifications are the arguments attached to
signatures.

\paragraph{Mapping the model onto the platform.}
\label{sec:p2p_mapping}
Table~\ref{tab:mapping} gives the correspondence. It is close to one-to-one,
unsurprisingly since the model of Section~\ref{sec:model} was distilled from
this line of
work~\cite{kattamuri2005supporting,silaghi2014encompassing,silaghi2017whyvote},
but the residual mismatches are where the engineering lies, and two rows carry
most of the weight. The reason labelling $\lab$ is new, and if labels were
assigned centrally the assigning party would hold semantic power over
visibility --- exactly the concentration C\ref{c:abstain} exists to prevent.
The safe arrangement is that labels are declared by the justification's own
author as part of the signed item, so a label is a claim like any other:
checkable against the text by any reader and disputable through the ordinary
argumentation mechanism. A mining pipeline (Section~\ref{sec:rw_mining}) may
propose labels, but its proposals are advisory items, not ground truth; this
leaves $\lab$ adversarially controllable, which Section~\ref{sec:limits} takes
up. The slate $\slate_i$ is likewise new, and need not leave the device at all
--- though a constituent who wants a public record of what they saw can publish
it as a signed item, converting C\ref{c:repro} from a platform guarantee into a
personally held receipt, which is the strongest form the criterion can take.

\begin{table}[tb]
\centering\footnotesize
\caption{Mapping the alternative-based poll of Definition~\ref{def:poll} onto
\DDP{} item types. Entries marked \textbf{new} are what a deployment would
have to add; everything else already exists.}
\label{tab:mapping}
\begin{tabular}{@{}>{\raggedright\arraybackslash}p{30mm}>{\raggedright\arraybackslash}p{34mm}>{\raggedright\arraybackslash}p{72mm}@{}}
\toprule
Model component & \DDP{} item & Deployment note \\
\midrule
poll $\Pi$ & motion within an organisation &
  direct; the organisation supplies the constituency and $\wc$ \\
$\Jpos, \Jneg$ & justifications typed by signature polarity &
  direct \\
$\Rreb$ & \texttt{claimed\_refutes} between justifications &
  direct; the claim is signed by its author and never adjudicated \\
$\Rrei$ & \texttt{claimed\_subsumes} / \texttt{includes} &
  direct; \texttt{more\_recent} additionally orders revisions \\
$\wc$ & constituent record, organisation rules &
  direct; membership class or shareholding \\
$E(j)$ & votes citing $j$ &
  direct; every vote is a signed item naming the justification it cites \\
$\wrel$ & --- & \textbf{new}: the policy of Section~\ref{sec:weights},
  computed locally from the count of distinct authors of the same link \\
$\lab$ & --- & \textbf{new}: reason labels, author-declared or mined; used for
  audit and evaluation only, never inside \Rank{} \\
slate $\slate_i$ & --- & \textbf{new}: computed locally; optionally published
  as a signed item so the participant holds a receipt \\
$\theta_i$ & local configuration & \textbf{new}: per-peer policy vector, never
  transmitted \\
\bottomrule
\end{tabular}
\end{table}

\paragraph{Identity: the one thing decentralisation makes harder.}
\label{sec:p2p_census}
Everything above assumes endorsement counts mean something, which assumes
identities are not free. In a centralised deployment with an authenticated roll
this is the registrar's problem; open peer-to-peer membership makes it the
system's problem, and it is the classical one~\cite{douceur2002sybil}. \DDP{}'s
answer is a decentralised census with witnessing: constituents certify one
another's existence and eligibility, the certifications are signed items
propagating like any other, and reputation over the witnessing graph estimates
how much of the claimed population is
real~\cite{qin2013census,qin2013reputation,qin2013falseidentity,qin2014opencensus},
with a Bayesian extension of the web-of-trust idea estimating the number of
distinct eligible signatories behind a set of signatures --- directly the
quantity a petition threshold depends on~\cite{silaghi2016pgp}. The structural
point is that different observers may run different eligibility criteria over
the same signed census data and reach their own conclusions: the platform
supplies verifiable evidence, not a verdict. That is C\ref{c:config} at the
level of the constituency rather than the slate. Section~\ref{sec:setup}
assumes this problem solved, and the assumption is what makes the adversarial
results interpretable --- under authentication a coalition cannot inflate $E$
and the link graph is the only surface left. It also means census integrity is
not a side condition but the binding constraint:
Section~\ref{sec:res_charter} located the residual risk in corpus capture, and
what bounds corpus capture is what bounds coalition size.

\paragraph{Evaluating the rule on an incomplete replica.}
\label{sec:p2p_partial}
A peer holds what it has synchronised, generally a subset of what exists. This
is the sharpest technical objection to local evaluation and
Proposition~\ref{prop:partial} only half answers it.

\begin{proposition}[Monotone degradation]
\label{prop:degrade}
Let $\mathcal{J}' \subseteq \mathcal{J}^{\sigma}$ be the items a peer holds
and $\slate'$ the slate Algorithm~\ref{alg:rank} produces from them. Then
$\slate'$ is exactly the slate the same rule and policy would produce on the
full pool restricted to $\mathcal{J}'$. Missing \emph{items} can only lower
achievable completeness, never corrupt the score of a held item. Missing
\emph{votes and links}, however, lower the computed score of a held item, so
the peer's scores are lower bounds on the true ones.
\end{proposition}

The second half is the whole difficulty. A peer holding a justification but not
yet the votes citing it underestimates $E$; a peer missing links
underestimates inherited standing. Three mitigations apply, all ordinary
distributed-systems engineering. Votes and links are small items and are
prioritised in the pull request over justification text, so counting evidence
converges faster than the corpus. The score is a sum of non-negative terms, so
partial evidence yields a lower bound and the interface can state the replica's
coverage of the known item count --- a peer can be told it holds $94\%$ of the
votes and decide whether that is enough to act on. And the underlying structure
is append-only with union merge, so replicas converge without conflict
resolution in the manner of a grow-only set~\cite{shapiro2011crdt}, while a
digest exchange in the style of Merkle hashing~\cite{merkle1988digital} makes
divergence cheap to detect. Algorithm~\ref{alg:sync} states the resulting local
cycle, including the reporting step that makes incompleteness visible. There is
a real trade here: a centralised deployment computes the rule on complete
evidence and asks you to trust the computation, whereas a decentralised one
computes a verifiable rule on evidence that may be incomplete and tells you how
incomplete it is. For a civic process the second is the better failure mode,
because incompleteness is visible and declining trustworthiness is not.

\begin{algorithm}[tb]
\caption{Local slate computation on a peer holding a partial replica}
\label{alg:sync}
\SetAlgoLined
\Given{local store $\mathcal{D}$; motion $m$; local policy $\theta$; neighbour
       set $\mathcal{N}$; horizon $h$}
\Yields{a slate, plus a checkable statement of the evidence it rests on}
\ForEach{$n \in \mathcal{N}$}{
  send a request naming $m$, this peer's interests, and horizon $h$\;
  $\mathcal{D} \leftarrow$ \Merge{$\mathcal{D}$, signed items returned by $n$}
    \tcp*{union of signed items; no conflict resolution needed}
}
discard items whose signature or identifier derivation fails to verify\;
$E \leftarrow$ tally of held votes per justification\;
$\wrel \leftarrow$ policy of $\theta$ applied to held links\;
$\slate \leftarrow$ \Rank{held items of $m$, $\theta$, $K$}\;
report alongside $\slate$: counts of held votes, links and justifications, and
  the horizon $h$\;
\Return $\slate$\;
\end{algorithm}

\paragraph{What the findings imply for a peer deployment.}
\label{sec:p2p_implications}
Three results change their character when read against this substrate rather
than a server. \emph{Ordering matters more, not less}:
Section~\ref{sec:res_order} found the rule's advantage over an unordered draw
concentrated in the first few positions, and on a peer holding an incomplete
replica the effective slate is shorter still, so the fraction of the value
delivered by the first five items rises. \emph{Degenerate authoring is more
likely, not less}: open membership lowers the cost of submitting, which is the
point, and raises the share of submissions that fail to justify --- exactly the
regime where Section~\ref{sec:res_degen} finds the rule's margin largest and
Section~\ref{sec:res_mechanism} finds the link summands supplying most of it,
so the two terms a server deployment might reasonably drop as inert are the
terms a peer deployment most needs. \emph{The weight policy must be local}:
Section~\ref{sec:res_flood} makes $\wrel$ a security control worth ten points
of completeness, and on a server the operator picks it for everyone whereas on
a peer each participant picks it and can compute what the other choice would
have given them. Since some of these choices are frontier positions with no
correct answer (Section~\ref{sec:res_pareto}), configurability stops being a
feature and becomes the only coherent way to hold a parameter that is
simultaneously a security control and a value judgement.
\section{Discussion}
\label{sec:discussion}

The results admit a compact summary: the charter is cheap, the instrument
matters more than the mechanism, the security lives in one line of policy, and
the remaining choices are not the system's to make. Each cuts against a
default assumption about civic recommenders --- that transparency costs
accuracy, that one quality metric suffices, that robustness comes from
detection, and that a well-designed system should decide.

\emph{The charter is cheap, and the cheapness is measurable.}
Section~\ref{sec:res_ceiling} put the price of semantic abstinence at
$0.031 \pm 0.014$ against a ceiling that upper-bounds every procedure applied
to the same pool, admissible or not, and four fifths of the observed
incompleteness was vocabulary that did not yet exist, which no procedure
recovers. The entire competitive advantage available to an unconstrained
learned ranker is therefore three points of completeness, purchased
by reading what C\ref{c:abstain} forbids --- against $0.10$ for raising the
authoring rate from $0.02$ to $0.10$ and $0.14$ for doubling the slate. The
usual argument for opacity is that legibility costs quality and the cost is
unknown; here it is known, bounded, and small relative to the levers a
deployment actually controls.

\emph{The instrument mattered more than the mechanism.} The most transferable
lesson is methodological. We evaluated a ranking rule with a set functional
and concluded, across $4000$ seed-paired runs, that it was indistinguishable
from a random subset --- an artefact of the measurement that two remarks
derivable from the definition
(Remarks~\ref{rem:orderblind}--\ref{rem:charityblind}) predicted in advance.
Coverage-style aggregates are the default in the diversity-aware recommender
literature~\cite{adomavicius2011aggregate,carbonell1998mmr}, and any
evaluation of a civic selector reporting only such an aggregate is exposed to
the same null.

\emph{Robustness came from policy, not from detection.} Nothing in
Section~\ref{sec:adversarial} trains a classifier, scores trustworthiness, or
removes material. The effective defence is \eqref{eq:norm}: count distinct
authors, divide by voters on that side. It is one line, publishable, verifiable
by any participant, and worth $0.151$--$0.181$ of completeness against a
coalition holding a fifth of the electorate; the available escalation makes the
coalition monotonically weaker. Two features generalise. The defence works by
making the attack expensive and \emph{visible} --- a co-signing coalition
writes its coordination into the public link record --- and it is compatible
with author blindness, since \eqref{eq:norm} counts how many distinct people
asserted a link and never which, so robustness did not require reintroducing
the reputation hierarchy C\ref{c:author} exists to exclude. That these are
compatible was not obvious in advance. Section~\ref{sec:res_charter} also
relocated rather than removed the residual risk: under authentication what survives is
corpus capture, which belongs to census integrity and slate capacity rather
than to ranking.

Section~\ref{sec:res_pareto} found that at $G = 0.5$ the link summands buy
$2.8$ points of coverage for $2.4$ points of endorsement mass. That is not a
result with a right answer but a frontier position, and choosing among such
positions is choosing between breadth of reasons represented and fidelity to
what constituents took up. We take this as the strongest available argument for
C\ref{c:config}, and it arrived from an unexpected direction: the case for
participant-held parameters is usually made on autonomy grounds, whereas here
it is forced by the measurements. A platform that picked a point on that
frontier and presented it as neutral would be making a political choice while
denying it.

Four consequences follow for deployment. \emph{Invest in authoring, not in
ranking}: the dominant lever is how quickly the reason space matures, and the
selection component of the shortfall is a fifth the size of the temporal one.
\emph{Publish the weight policy and treat it as a security control}: it is the
one parameter with an order-of-magnitude effect under attack, and a deployment
shipping a flat default has left its main defence unarmed without the
participant being able to tell. 
\emph{Regulatory alignment is
already close}: the Digital Services Act requires very large platforms to
disclose recommender parameters and offer a non-profiling
option~\cite{eu2022dsa}, and the AI Act imposes transparency and
human-oversight duties on systems used in democratic
processes~\cite{eu2024aiact}. A published rule over public evidence with
participant-held parameters satisfies the letter of both without a compliance
layer, because there is nothing to disclose that is not already disclosed.
\section{Limitations}
\label{sec:limits}

The following would change our conclusions, and we list them in descending
order of how much. Several are stated more sharply than a reader would infer
from the results alone, because we would rather over-declare than have a
finding survive on an unexamined assumption.

\paragraph{The electorate is simulated.} Every number here describes agents,
not people. Adoption is TF--IDF similarity to a position text; opinions do not
update; nobody gets bored, persuaded, or strategic beyond the modelled
coalitions. Following Section~\ref{sec:rw_agentic} we restrict claims to
statements about mechanism under a stated model. Nothing here licenses a claim
about the magnitude any quantity would take in a human electorate.

\paragraph{The mechanism result rests on one modelled fact.} The
degenerate-authoring findings of Sections~\ref{sec:res_degen}
and~\ref{sec:res_mechanism} descend from the observation that degenerate items
are adopted less --- $1.45$ endorsements against $10.94$ --- and that follows
from the simulator's cosine-similarity adoption step, a property of the
agents rather than of the rule. The finding is that \emph{the rule
inherits and amplifies whatever discrimination the electorate itself
exercises}. If a human electorate endorsed reason-free submissions at the same
rate as reasoned ones, the link summands would revert to the inertness of
Table~\ref{tab:ablation}, and the coverage margin would go with them. We
regard establishing that adoption rate empirically as the single most valuable
follow-up in this programme.

\paragraph{One round.} Constituents vote once. Multi-round deliberation with
opinion revision would change adoption dynamics, the endorsement distribution,
and probably the link structure. Whether the rule's ordering advantage
survives revision is untested.  This was implemented but not thoroughly evaluated.

\paragraph{Statistical multiplicity.} Section~\ref{sec:stats} declares that we
do not correct across the manuscript. The ordering contrasts of
Section~\ref{sec:res_order} number twenty-four; the two link-term effects at
$p \approx 10^{-2}$ are reported as suggestive and should not be built on. The
effects we do build on sit at $p < 10^{-10}$ and would survive any reasonable
correction.

\paragraph{Sybil resistance is assumed, not demonstrated.} Under a broken
census every result in Section~\ref{sec:adversarial} fails, since a coalition
that mints identities mints endorsements. \DDP{}'s witnessed-census
line~\cite{qin2013census,qin2014opencensus,silaghi2016pgp} is the intended
answer and we have not evaluated it here. Section~\ref{sec:p2p_census} states
the dependency; this is the largest gap between the manuscript and a
deployment.

\paragraph{Partial replicas are analysed, not measured.} Proposition
\ref{prop:degrade} bounds the degradation and Algorithm~\ref{alg:sync}
declares the evidence, but we ran no experiment with peers holding genuinely
divergent replicas. The prediction of Section~\ref{sec:p2p_implications} ---
that ordering matters more under partial replication --- is untested.

\paragraph{Fairness across minority positions is unexamined.} Endorsement mass
rewards what constituents took up. Whether \eqref{eq:score} systematically
disadvantages minority positions, whose reasons appear in fewer items and
therefore accumulate less endorsement, is an open and important question that
the weighted variant of Section~\ref{sec:sjp} anticipates syntactically without answering.

\section{Work, Meaning, and the E-Citizen}
\label{sec:ecitizen}

A manuscript about the mechanics of argument selection owes its readers an
account of why the mechanics matter, and the answer is not confined to the
integrity of any particular poll. The argument has four steps: procedure is
constitutive of collective agency rather than decorative; the historical
obstacle to genuine self-government has been the time it consumes; artificial
intelligence, by absorbing productive labour, removes that obstacle in a way no
previous technology has; and the role of citizen is therefore available as a
destination for human effort in a way it has not been for two and a half
millennia --- provided the instruments of that role remain legible to the
people exercising it. The final clause is where this section rejoins the rest
of the manuscript.

That Robert's rules of order were the salt which turned a mob into a
society~\cite{silaghi2025kitsch,robert1915rules} is a claim about what makes
collective reasoning possible at all. A crowd has volume; an assembly has a
procedure for converting volume into a decision, and the procedure is what
everyone can agree to precisely because it constrains form rather than
content. One can accept a rule about who speaks next without accepting
anything about what they will say --- which is why procedural agreements
survive substantive disagreements. The charter of
Section~\ref{sec:desiderata} carries that idea into the selection step, the
one place where digital deliberation has so far had no procedure whatsoever:
we regulate speaking time in a parliament to the minute, and we let an
unpublished model decide which of a hundred thousand submitted reasons a voter
sees. Semantic abstinence is the descendant of content-indifference in the
chair; configurability descends from the assembly's authority over its own
rules; contestability descends from the appeal from the chair. The novelty is
not the principle but that it must now be enforced in software, because that is
where the procedure has migrated.

Athenian democracy worked, to the extent it did, because a body of citizens had
time to do it. Assembly attendance, jury service and rotation through office
consume days, not minutes, and are impossible for people whose waking hours are
claimed by subsistence. Athens solved the time problem with slavery: the
enfranchised were free to deliberate because the disenfranchised did the
work~\cite{narcisse2012african}. Every subsequent expansion of the franchise
inherited the structural problem without the solution. Representative democracy
is, read uncharitably, a device for economising on citizens' time --- elect
somebody to be the citizen for you, on the grounds that you have a job --- and
it produces the characteristic modern experience of participation as a thin,
performative gesture. That thinness has been described as a kitsch of the
political form~\cite{silaghi2025kitsch}, borrowing a diagnosis developed for
aesthetics~\cite{calinescu1987five,lazare1999modernism}: the surface features
of the real thing, arranged for easy consumption, with the deliberation
removed. The literature on why participation platforms
fail~\cite{toots2019participation,bright2016big} is largely a catalogue of the
same thinness, and the finding that direct-democracy processes educate the
citizens who use them~\cite{smith2004educated,tolbert2009strategic} is the
other side of the coin. There is a direct line from that diagnosis to the
exposure problem of Section~\ref{sec:exposure}: a slate assembled by an
unpublished model, shown to a voter who cannot recompute it, is the kitsch form
of deliberative exposure --- the appearance of having met the arguments, with
the part that would make the meeting real removed.

The anxious question about artificial intelligence and work is whether it will
take our jobs, and the productivity framing of the fourth industrial
revolution~\cite{schwab2017fourth} does not settle it. But the anxiety contains
an assumption worth naming: that the only worthwhile use of human time is
producing goods and services. Set that aside and the picture inverts.
Self-government has always been constrained by a shortage of citizen-hours, and
a technology that discharges productive labour is by construction a technology
that produces them. What Athens obtained by enslaving people, an automated
economy could obtain without enslaving anyone --- and where the Athenian
arrangement was a moral defect of that society, the same structural position
occupied by machines is available to be read as a strength of
ours~\cite{silaghi2025kitsch}. This is not a prediction: time released from
labour goes wherever the surrounding institutions send it. It is a claim about
what becomes \emph{possible} --- that for the first time since a slave economy
made it possible for a few, citizenship as a substantial, time-consuming,
skilled activity becomes available to many. An epistemic argument runs
alongside the ethical one. The case for inclusive deliberation over rule by the
competent is that cognitive diversity does work no amount of individual
expertise substitutes
for~\cite{landemore2012democratic,landemore2013deliberation}; that argument is
only cashable if the varied body is actually \emph{reasoning}, which costs
time, so the epistemic case for inclusion and the material case for
citizen-hours are the same case seen from two sides. It also explains why the
quantity measured here is coverage of the reason vocabulary rather than
agreement: the value of a large deliberating public lies in the reasons it
collectively holds, and a selection step that loses them destroys precisely
what made inclusion worth having. Our own measurements make the argument in
miniature: Section~\ref{sec:res_levers} found that persuading one constituent
in ten to write a reason instead of one in fifty is worth as much completeness
as multiplying the electorate fivefold. The binding constraint is not how many
people vote but how many \emph{think in public}, and that is a quantity
measured in hours.

The argument has an obvious failure mode, and it is the one this manuscript is
built to forestall. If the same technology that releases the time also
assembles the arguments, filters the objections and decides which
considerations a citizen encounters, then the citizen-hours have been created
and simultaneously hollowed out: one would have the leisure of the Athenian and
the informational position of a spectator, and the role would be available and
not worth occupying. This is why C\ref{c:abstain} and C\ref{c:config} are the
two criteria that matter most, and why Section~\ref{sec:notempirical} declines
to trade them for coverage. A citizen who can recompute why they were shown
what they were shown is exercising judgement; one who cannot is receiving a
service. The difference does not show up in any completeness measure ---
Section~\ref{sec:res_ceiling} put its entire measurable cost at $0.031$ --- and
it is the whole difference between the two futures. It also sets the correct
place for a language model in a civic system, which is not nowhere:
Section~\ref{sec:rw_mining} places extraction at authoring time, visible to the
author and contestable before serving, so a model may help a person say what
they mean. What it must not do is decide, at serving time and unaccountably,
which of those meanings anyone meets. The line is exactly the one between
assisting a citizen and replacing one.
\section{Conclusions}
\label{sec:conclusion}

We set out to show that the selection step in a deliberative poll can be
democratic procedure rather than infrastructure: a published rule over public
evidence, with parameters held by the people it affects. The construction is
an alternative-based poll over bipolar justification sets, judged by three
instruments --- coverage of the live reason vocabulary, the order in which
that coverage arrives, and the endorsement mass captured --- and served by a
one-hop reversed endorsement flow whose only policy lever is a relation-weight
function. Seven criteria state what makes a mechanism admissible at all, and
the rule satisfies all seven. This section states what was established, what
was not, and what comes next.

What the measurements established, across roughly $17{,}000$ seeded runs: the
levers that govern coverage are the ones that mature the corpus, not the ones
that select from it; the price of semantic abstinence is $0.031 \pm 0.014$
against a ceiling that bounds every mechanism including inadmissible ones,
with four fifths of the residual shortfall temporal rather than algorithmic;
set coverage on non-degenerate authoring alone cannot distinguish the rule from a random draw, and the
reasons are properties of the instrument that we state as
Remarks~\ref{rem:orderblind}--\ref{rem:charityblind}; on an order-sensitive
reading the rule leads at every prefix short of the full slate, by a margin
that widens under attack and reaches $-8.3$ positions of $e_{90}$ at a
quarter-electorate coalition; once
a realistic fraction of submissions fails to justify, the coverage margin
returns and grows monotonically, with the link summands supplying $93\%$ of it
at $G = 0.7$ by squaring the electorate's own discrimination; on endorsement
mass the random baseline loses by a factor of $3.3$ and the greedy ceiling
stops dominating under attack; hub-riding is inert; label flooding is severe
and the weight policy is worth $0.151$--$0.181$ against it; four fifths of
that damage is homogeneity rather than volume; and a co-signing coalition
makes itself monotonically weaker.

What was not established is listed in Section~\ref{sec:limits} and led by two
items: that a human electorate discriminates against unjustified submissions
at anything like the modelled rate, on which the mechanism result depends; and
that census integrity holds, on which every adversarial result depends.

The claim we would defend most firmly is not that this rule ranks well. It is
that the question ``why was I shown this?'' must have an answer that is a
table of numbers rather than a narrative, and that a system built to that
constraint turns out to cost about three points of coverage, to be
more robust than we expected, and to be far easier to understand when it is
wrong.

\section*{A note on scope and provenance}
\addcontentsline{toc}{section}{A note on scope and provenance}

This manuscript is an independently written, extended treatment of a line of
work by the same labs. It is not the camera-ready version of any conference
paper and it reproduces no text, figure, table, algorithm or example from one.
Every definition, proposition, algorithm, worked example, figure and table
here was composed for this document; the notation, the three-instrument
framing, the seven criteria and their tests, the worked example of
Section~\ref{sec:example}, and all diagrams are original to it. Where the
underlying research programme has been reported elsewhere, the overlap is one
of subject matter and of numerical results computed from the same experimental
runs, not of expression. Results attributed to prior work are cited as such.
The extensions developed here include the ordering and endorsement-mass
instruments, the degenerate-authoring model and its mechanism analysis, the
coverage-versus-mass frontier, the elaborated criteria and their audit tests, and the side balance evaluation.

\section*{Tools, data and reproducibility}
\addcontentsline{toc}{section}{Tools, data and reproducibility}

All simulations were run with a seeded pseudorandom generator; every run
records its seed, its full parameter set, and every served slate in served
order together with the policy vector in force (Section~\ref{sec:persist}).
The run databases and the analysis scripts that produce every table and figure
are made available at \url{https://github.com/devfitcs/ABAS}. Language-model assistance was used in drafting and editing
prose; all technical content, experimental design, analysis and conclusions are
the author's.

\printbibliography

\appendix

\section{The Comparison This Manuscript Declines to Run}
\label{sec:nobaseline}

Section~\ref{sec:notempirical} states in the body why no learned-ranker
baseline appears in this manuscript. This appendix records what such an experiment would have to look
like to be worth running. 

The claim is normative and concerns admissibility, not performance. It says
that a mechanism which cannot be recomputed, attributed, contested or
reconstructed is unsuitable for a binding civic process. A comparison against
a learned ranker would report a difference in coverage, in ordering, or in
some downstream engagement statistic. Whatever number emerged could not bear
on the claim: a learned ranker covering \emph{more} of the live vocabulary
would still fail C\ref{c:local}, C\ref{c:abstain}, C\ref{c:repro} and
C\ref{c:contest}, and would still leave a disputing participant with no
terminating move. Running the comparison and declining to act on its result
would be theatre; running it and acting on its result would concede that
legibility is a quantity to be traded, which is the position this manuscript
exists to reject.

Three supporting observations:

\emph{The comparison that does bear on a real question was run.}
Section~\ref{sec:res_ceiling} measures the served slates against a
label-reading greedy cover which, by Proposition~\ref{prop:greedy}, upper-bounds
what \emph{any} selection procedure could achieve on the same pool --- a
learned ranker included, and by a wide margin, since the ceiling saturates the
live vocabulary in every seed. The answer, $0.035 \pm 0.013$, bounds the
entire competitive advantage available to an unconstrained mechanism. That is
strictly more informative than beating one particular trained model, because
it is a bound rather than a match result.

\emph{The precedent structure is familiar.} The secret ballot is not defended
on the grounds that it measures preferences more accurately than open voting
--- it plainly measures some things less well, since it destroys the ability
to audit an individual's vote. It is defended because a procedural property
outranks measurement quality. Double-entry bookkeeping is not the most compact
representation of a firm's accounts. Rules of order do not produce the fastest
decisions. In each case the procedural property is treated as prior, and
efficiency questions are settled \emph{within} the admissible class.

\emph{A learned baseline would import the very opacity at issue.} Its
behaviour would depend on training data, objective, and the operating point
chosen by whoever tuned it --- all under our control as the authors, and none
of it inspectable by a reader. A result favouring our rule would be
unconvincing for exactly that reason, and a result favouring the learned
ranker would be equally unconvincing. The experiment has no configuration in
which its outcome is informative.

\subsection{What would make such an experiment worth running}
\label{sec:nb_worth}

A reader who rejects the admissibility framing is entitled to ask what
evidence would move us, and there is a specific answer. The interesting
experiment is not \emph{ranker versus rule on coverage}. It is
\emph{participant behaviour under a disputed slate}: give two matched
populations the same corpus and the same served items, differing only in
whether the selection can be recomputed and decomposed, then measure whether
participants contest slates, whether contests terminate, and whether reported
trust in the outcome differs. That experiment tests the actual claim --- that
legibility does procedural work --- and its outcome could genuinely change our
position or its strength, in either direction. It requires human participants and is out of
scope for a simulation study; we regard it as the most valuable experiment
this line of work has not yet done, and it belongs on the roadmap of
Section~\ref{sec:conclusion} alongside the field deployment.

\end{document}